\pdfoutput=1
\documentclass{article}

\usepackage{microtype}
\usepackage{graphicx}
\usepackage{subfigure}
\usepackage{booktabs} 
\usepackage{multirow}
\usepackage{hyperref}

\usepackage[accepted]{dapao}

\usepackage{amsmath}
\usepackage{amssymb}
\usepackage{mathtools}
\usepackage{amsthm}

\usepackage[capitalize,noabbrev]{cleveref}

\theoremstyle{plain}
\newtheorem{theorem}{Theorem}[section]
\newtheorem{proposition}[theorem]{Proposition}

\theoremstyle{definition}
\newtheorem{definition}[theorem]{Definition}

\theoremstyle{remark}

\usepackage[textsize=tiny]{todonotes}

\usepackage[T1]{fontenc}
\usepackage{breakurl}
\usepackage{soul, color, xcolor}

\usepackage{amsfonts}
\usepackage{makecell}
\usepackage{algorithm}
\usepackage{algorithmic}

\begin{document}

\twocolumn[
\icmltitle{Robust Watermarks Leak: Channel-Aware Feature Extraction Enables Adversarial Watermark Manipulation}




\icmlsetsymbol{corr}{*}

\begin{icmlauthorlist}

\icmlauthor{Zhongjie Ba}{sch}
\icmlauthor{Yitao Zhang}{sch}
\icmlauthor{Peng Cheng}{sch,corr}
\icmlauthor{Bin Gong}{sch}
\icmlauthor{Xinyu Zhang}{sch}
\icmlauthor{Qinglong Wang}{sch}
\icmlauthor{Kui Ren}{sch}

\icmlauthor{}{sch} The State Key Laboratory of
Blockchain and Data Security, Zhejiang University, Hangzhou, China
\end{icmlauthorlist}




\vskip 0.3in
]




\renewcommand{\thefootnote}{}
\footnotetext{\textsuperscript{*}Corresponding author}

\begin{abstract}
Watermarking plays a key role in the provenance and detection of AI-generated content. While existing methods prioritize robustness against real-world distortions (e.g., JPEG compression and noise addition), we reveal a fundamental tradeoff: such robust watermarks inherently improve the redundancy of detectable patterns encoded into images, creating exploitable information leakage. To leverage this, we propose an attack framework that extracts leakage of watermark patterns through multi-channel feature learning using a pre-trained vision model. Unlike prior works requiring massive data or detector access, our method achieves both forgery and detection evasion with a single watermarked image. Extensive experiments demonstrate that our method achieves a 60\% success rate gain in detection evasion and 51\% improvement in forgery accuracy compared to state-of-the-art methods while maintaining visual fidelity. Our work exposes the robustness-stealthiness paradox: current "robust" watermarks sacrifice security for distortion resistance, providing insights for future watermark design.
\end{abstract}


\section{Introduction}
Watermarking is a well-recognized technique in the era of artificial intelligence-generated content (AIGC) for the purpose of content provenance and deepfake detection~\cite{jiang2024watermark,jiang2023evading}. Although the emergence of generative AI products such as Midjourney~\cite{Midjourney}, DALL-E~\cite{DALLE3}, and Sora~\cite{Sora} has lowered the bar of access to cutting-edge AI technology for ordinary users, it also benefits misleading content generation and the spread of misinformation for evildoers~\cite{fakenews1}. To address this, leading IT companies providing AIGC services, including OpenAI, Alphabet, and Meta, have committed to deploying watermark mechanisms in their products to make the technology safer~\cite{Whitehouse}. Watermark's capabilities of tracing, provenance, and detection of AIGC content facilitate the protection of content creators' intellectual property and the reputation of companies.

\begin{figure}[!t]
    \centering
    \includegraphics[width=\linewidth]{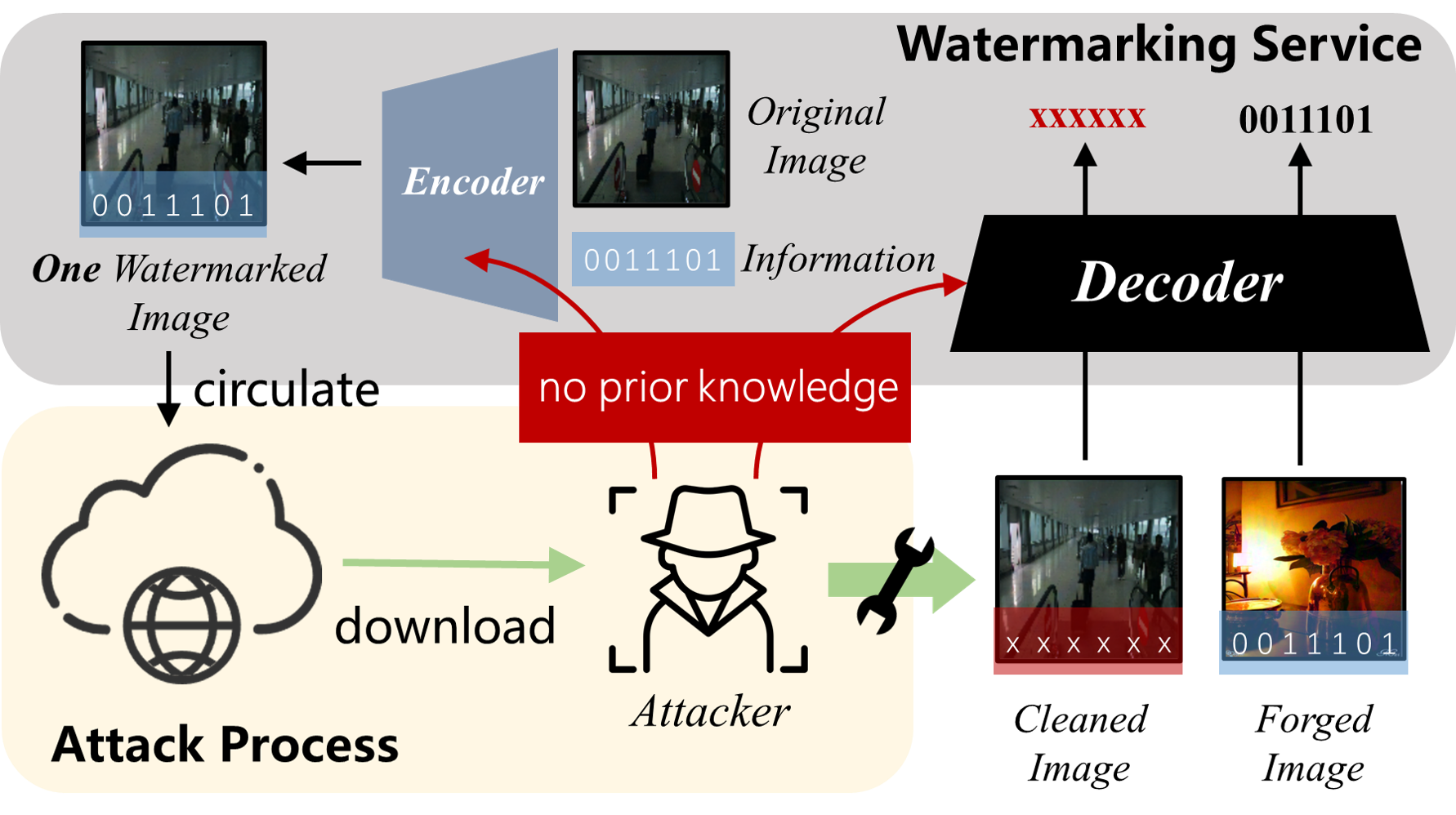} 
    
    \caption{Demonstration of our attacks. An attacker can perform watermark removal and forgery attacks with only one watermarked image without knowledge about the underlying watermarking systems. The attacker is free of copyright violation accusations as the extracted watermark is incorrect; the attacker can spread fake news by forging the watermark of an authoritative media.} 
    \label{fig:intro}
    \vspace{-1mm}
\end{figure}

Watermarking methods improve robustness against distortions, enhancing practicality in real-world applications. When images spread across social media platforms, they often encounter distortions such as compression, noise addition, and screen-shooting~\cite{PIMoG}. Improving watermark robustness against these operations ensures the watermark remains detectable after Internet circulation.

Apart from distortions, watermarks are susceptible to adversarial attacks, mainly detection evasion and forgery attacks. These attacks can spread Not-Safe-for-Work (NSFW) content, copyright violations, and reputation undermining. An attacker can remove the watermark from an artist's imagery, thus evading the watermark detection and claiming ownership of the asset for commercial purposes. An attacker can also analyze the watermark pattern hidden in a watermarked image, extract and transfer it to a clean image depicting illegal content to make fake news trustworthy and damage a specific user or an organization's reputation. 

Studying the security of watermarking schemes is important due to its wide deployment in the AIGC era. We mainly discuss \textbf{the security of watermark schemes that incorporate a distortion layer to strengthen their robustness.} These watermarking methods are promising for practical use and exhibit better resilience when encountering attacks. 

Most watermark security research focuses on evading detection, and studies regarding forgery attacks are on the rise. \textbf{Detection Evasion:} We categorize relevant literature into adversarial example-based and reconstruction-based attacks. The former typically requires a large amount of watermark data—often including pairs of watermarked and original images—for perturbation training~\cite{lukas2024leveraging,yang2024steganalysis, saberi2023robustness} or direct access to the watermark detector for querying, sometimes even knowledge of the decoder~\cite{lukas2024leveraging}; the latter introduces noticeable modifications to the carrier image content during reconstruction~\cite{saberi2023robustness, zhao2023invisible, Kassis2024Unmarker}. \textbf{Watermark Forgery:} Watermark forgery represents a new research frontier, with its attack methodologies still in the developmental stages. The existing method relies on impractical assumption~\cite{kutter2000watermark}, requires access to the watermarking encoder~\cite{saberi2023robustness}, or demands substantial imagery that features the target watermark, which is ineffective against dynamic watermarks that incorporate time-sensitive nuance.

In summary, existing detection evasion methods are computationally expensive, rely on strong assumptions, or significantly alter the original image's semantics. Meanwhile, forgery techniques are either impractical or high-cost, and the overall effectiveness of both attacks remains limited.

In this paper, we identify a key weakness in watermarking systems designed to be robust against distortions and propose a highly effective attack framework that exploits this vulnerability. Our observation is that \textbf{these systems enhance robustness by increasing watermark information redundancy. However, this heightened redundancy inadvertently makes watermark leakage easier} (see Sec.~\ref{sec:pilot} for details). Based on the observation, we \underline{\textbf{D}}elve into the \underline{\textbf{A}}spect of the \underline{\textbf{PA}}radox \underline{\textbf{O}}f Robust Watermarks and propose the \textbf{DAPAO} attack.

DAPAO attack is a framework that effectively extracts the digital fingerprint from watermarked images, capable of both evading watermark detection and forging watermark on clean images. We successfully identify the watermark leakage using \textbf{only one watermarked image} in the \textbf{no-box setting}, meaning no queries to the target watermarking system are required. Our method significantly outperforms state-of-the-art (SOTA) approaches in both watermark removal and forgery while preserving visual fidelity and semantic integrity. To extract the watermark-related feature, we utilized a typical neural network to extracting image features, representing in the form of multiple channels. We identify critical channels that have a bias toward watermark features and optimize learnable variables to align with the watermark characteristics and the semantics of the carrier image. These trained variables can be subtracted from the watermarked image for detection evasion and added to clean imagery for spoofing attacks. However, two key challenges arise: 1) accurately identifying the channels containing watermark features; 2) effectively forging semantic watermarks, where watermark features are deeply coupled with the carrier image’s semantics. 

To address these challenges, we propose an aggregation-based locating algorithm to identify key channels automatically, then optimize the learnable variable using weight-based loss function. The resulting variable facilitates both attack goals. For semantic watermarking, we propose to bridge the semantic difference between the carrier and the target images with a new optimization paradigm, successfully forging semantic watermarks. We conduct comprehensive evaluations to validate the effectiveness of DAPAO, benchmark its performance against existing methods, and perform ablation studies to analyze the contributions of each component in our framework.

\textbf{Summary of contributions.} In this paper, we make the following contributions:

\begin{itemize}
    \item We reveal the robustness-stealthiness paradox of watermark systems: Schemes improve watermark information redundancy to boost robustness against distortions, which results in watermark feature leakage that can be leveraged by attackers.
    \item Leveraging this observation, we propose DAPAO, a novel attack framework capable of both watermark removal and forgery against SOTA robust watermarking schemes. Our method requires only a single watermarked image for extraction and operates in a no-box setting. Additionally, we introduce a new approach to overcome the challenge of forging semantic watermarks, an understudied problem.
    
    \item Extensive experimental results demonstrate that our framework achieves notable improvement in attack performance over related work, with a 60\% success rate improvement in watermark detection removal and a 51\% improvement in forgery accuracy. 
    
\end{itemize}


\section{Background}~\label{sec:background}
This section provides the necessary background for understanding our attack framework.

\begin{figure}[!t]
    \centering
    \includegraphics[width=\linewidth]{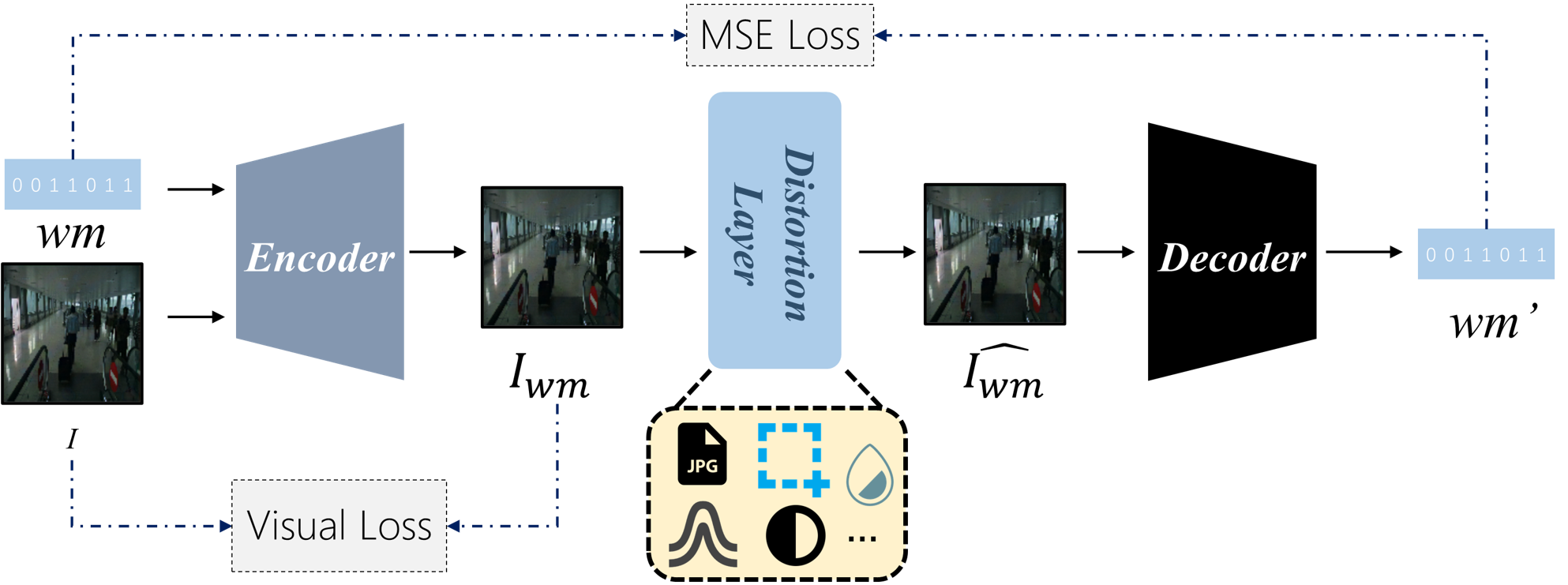} 
    
    \caption{Illustration of learning-based watermarking methods.} 
    \label{fig:watermark}
    \vspace{-3mm}
\end{figure}

\subsection{Image Watermarking}
Image watermarking includes injection, extraction, and verification. During watermark injection, an encoder $\mathcal{E}(\cdot,\cdot)$ receives the identification information $wm$ (``0011011" in Figure~\ref{fig:watermark}) and an original image $I$ as input and generates a watermarked image $I_{wm}$ with the key information embedded. During watermark extraction, the decoder $\mathcal{D}_{wm}(\cdot)$ extracts the identification key $wm'$ from the watermarked image $I_{wm}$ and then matches it 
 with $wm$ to verify whether the target watermark exists in the image.

\textbf{Non-learning-based and Learning-based Watermarking.} Non-learning-based methods build the encoder and decoder based on heuristics~\cite{jiang2023evading}. Learning-based methods deploy neural networks for the encoder and decoder, whose parameters are trained with deep learning techniques. Generally speaking, learning-based methods exhibit more robustness against distortions. In particular, they can incorporate a distortion layer before the decoder to mimic possible distortions and perform adversarial training(as shown in Figure~\ref{fig:watermark}) , causing the results of decoding a processed watermarked image $\hat{I_{wm}}$ to be identical to the one without experiencing distortions~\cite{jiang2023evading}.    

\textbf{Post-processing and In-processing Methods.} Post-processing watermarking adds a watermark to an image post its generation, following the same process of watermarking a real image~\cite{chopra2012lsb, DWT-DCT, stegastamp}. In contrast, in-processing watermarking embeds the identification message during the image generation process~\cite{YU1, yu2021responsible}. 

\subsection{Detection Evasion and Watermark Forgery}
Detection evasion means an attacker modifies a watermarked image to remove or disrupt the embedded watermark, causing the decoded bit string to deviate from the original identification information.

Watermark forgery involves extracting the watermark information $wm$ from the watermarked image $I_{wm}$ and embedding it into another non-watermarked image $I'$ to generate $I'_{wm}$, passing the verification of the watermark detector.
\section{Problem Formulation}

\begin{figure}[!t]
    \centering
    \includegraphics[width=\linewidth]{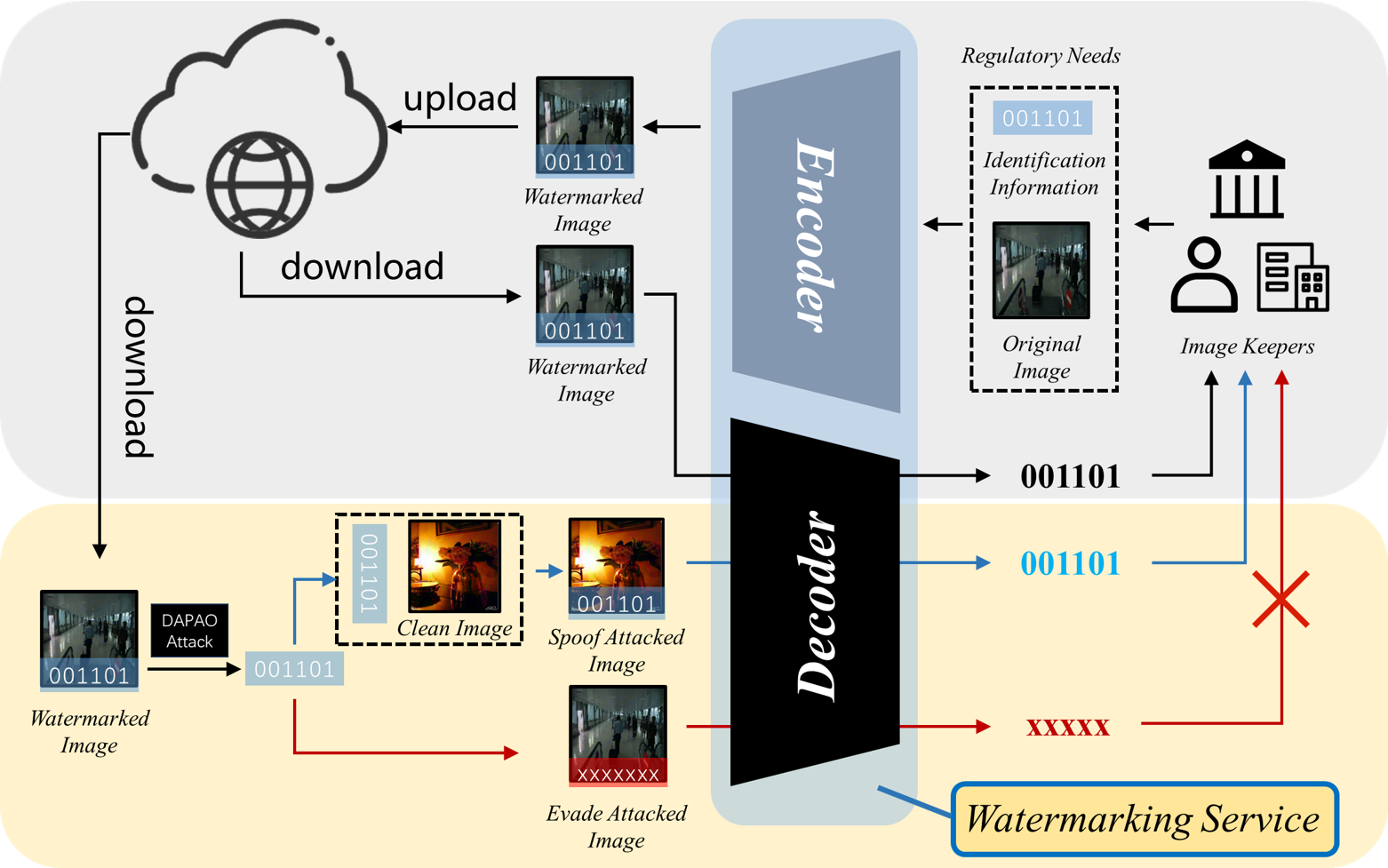} 
    
    \caption{Typical watermarking application and security threats. Organizations and individuals use watermarking services to embed watermarks into images for purposes such as copyright protection or content regulation. When image ownership verification is required, the watermark is extracted and matched through the watermarking service. However, attackers can apply carefully designed post-processing techniques to remove or forge the watermark.} 
    \label{fig:models}
\end{figure}

\subsection{System Model}
Figure~\ref{fig:models} illustrates the use case of a typical watermarking system. The process consists of the stages of watermark injection (encoding), data circulation, and watermark extraction (decoding), as shown in the gray portion of Figure~\ref{fig:models}. We primarily consider the post-processing watermarks. The three parties involved include \emph{users/organizations}, \emph{the verifier}, and the\emph{the attacker}.


\textbf{Users/service providers.} Users would like to use watermarking service before posting images online via social platforms to protect copyright. Alternatively, a service provider wants to mark all imagery generated by its own products, ensuring content provenance.

\textbf{The verifier.} To verify if an image contains the watermark, the verifier downloads target images from the Internet, decodes the image to extract watermark information, and then verifies the extracted one with the identification information. 

\subsection{Attacker's Goals} An attacker has two types of objectives. First, he would like to use an image without attributing it to the creator; therefore, he needs to evade the detection of watermarks. Second, he would like to improve the credibility of a fake image; therefore, he needs to 
forge a watermark related to an official account.

\subsection{Attacker's Capability} The attacker can download watermarked images uploaded by the victim, perform watermark removal or watermark spoofing on a clean image. Notably, we assume three realistic limitations: 1) The attacker \textbf{neither have knowledge} about the target watermarking system (i.e., encoder and decoder), \textbf{nor can he query the system}; 2) The attacker cannot obtain the original image; 3) The attacker must tackle watermark methods that are robust against distortions.

\section{DAPAO Attacks}
In this section, first, we present the feasibility study, demonstrating our observation of information leakage in robust watermarks. Next, we provide the theoretical analysis for method validation. Last, we introduce evasion and forgery attacks based on the observation.


\subsection{Feasibility Study}\label{sec:pilot}
We empirically find that learning-based robust watermarking systems counteract distortion effects (e.g., compression) by expanding the regions where the watermark pattern is embedded or amplifying its magnitude, ensuring that the watermark remains detectable. Beyond the encoding process, these systems also train the watermark decoder to enhance extraction effectiveness, effectively increasing the model's attention to watermark signals.

We conduct a feasibility study to explore: \emph{If the strengthened watermark results in leakage that can be captured from images using a feature extraction network?} We embed watermarks in multiple images with the same robust watermarking algorithm and then input these watermarked images into a feature extraction network.

\begin{figure}[!t]
    \centering
    \includegraphics[width=\linewidth]{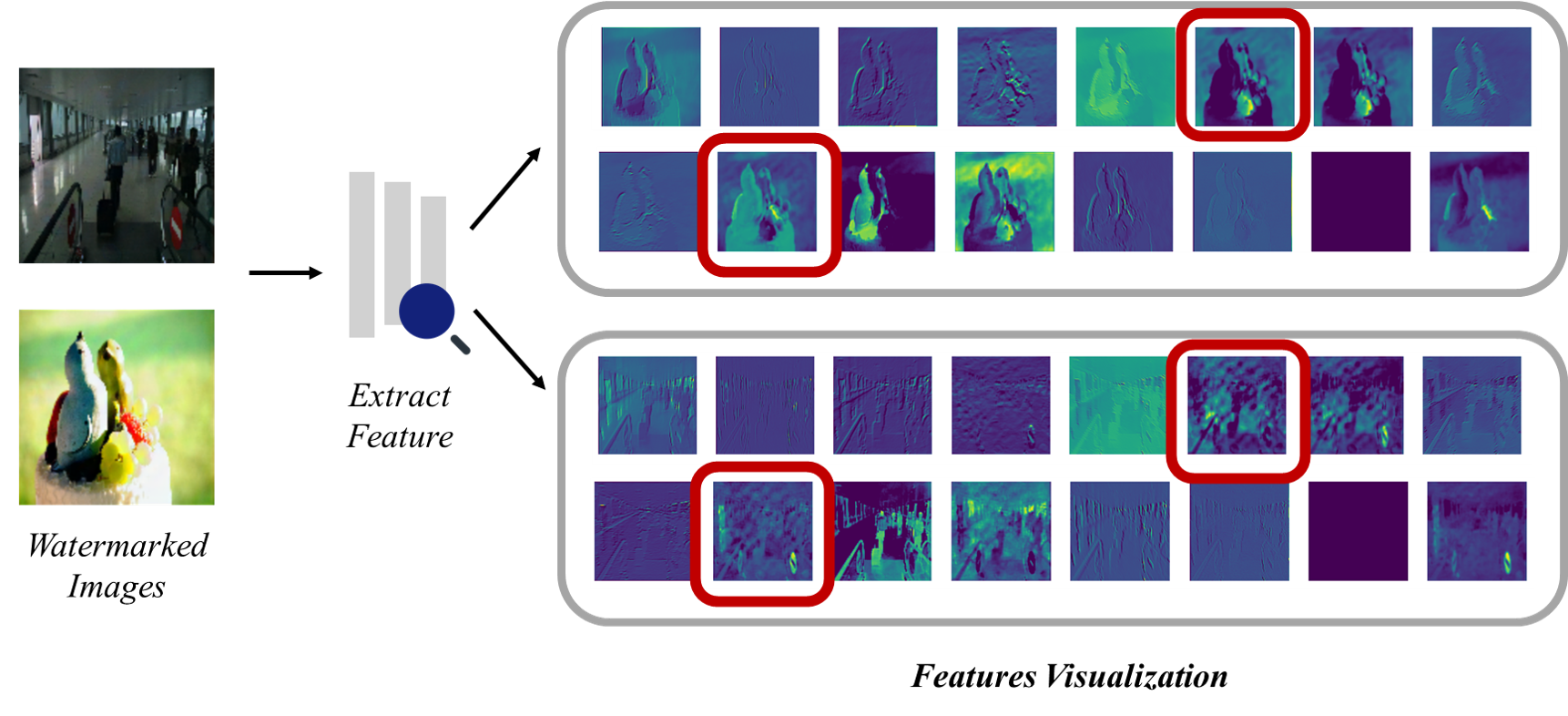}   
    \vspace{-6mm}
    \caption{Demonstration of our feasibility study.}
    \label{fig:feasibility}
    \vspace{-3mm}
\end{figure}

As shown in Figure ~\ref{fig:feasibility}, we found that:
\begin{itemize}
    \item The multi-channel features obtained after feature extraction can capture patterns not easily noticeable by the human eye.
    \item These patterns are similar across different images.
    \item Not all features contain such leakage information.
\end{itemize}

The results shed light on learning watermark characteristics from distinguished patterns probably related to the watermark.

\begin{figure}[!t]
    \centering
    \includegraphics[width=\linewidth]{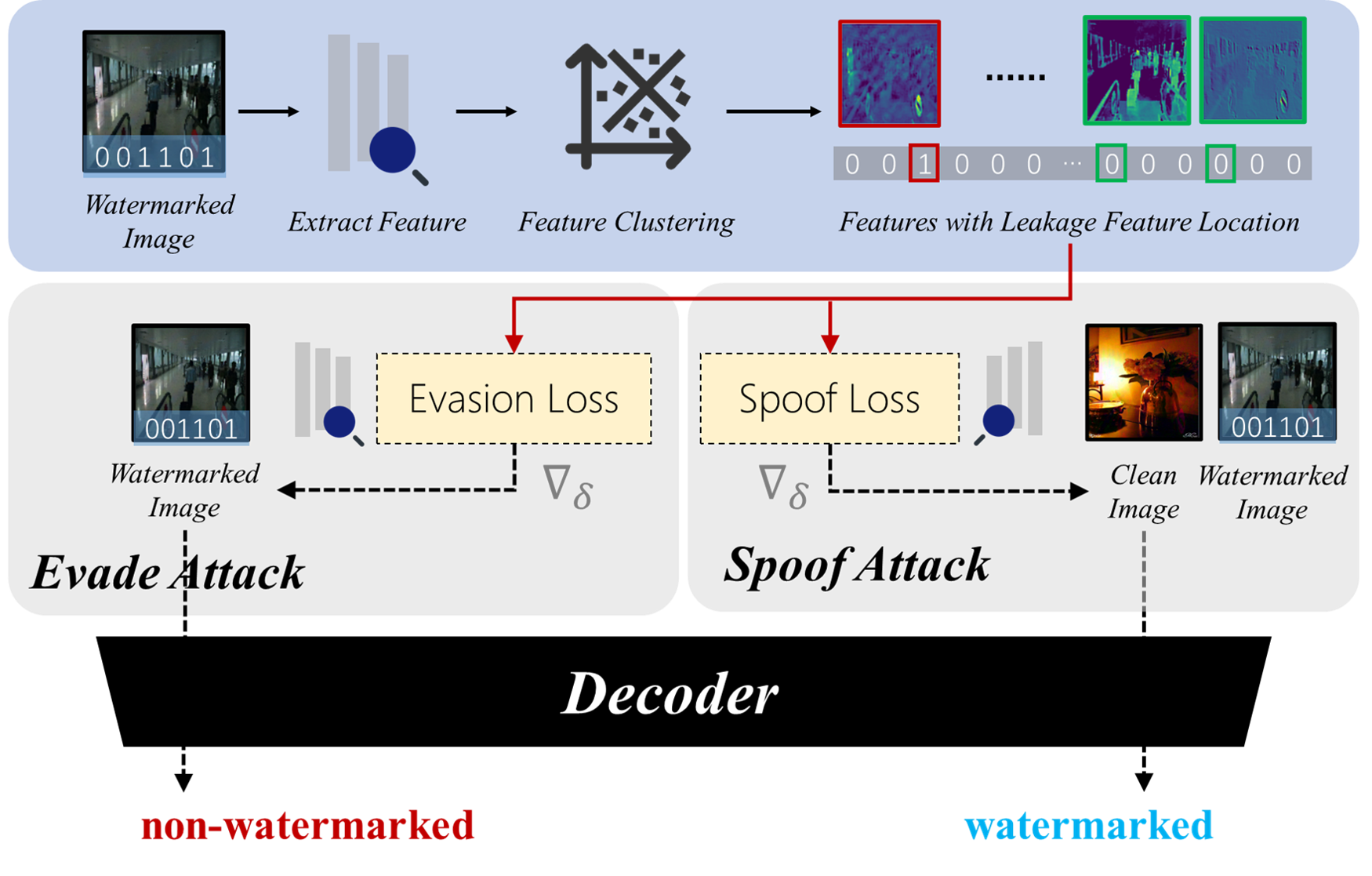} 
    
    \caption{An overview of our attack.} 
    \label{fig:method-overview}
    \vspace{-3mm}
\end{figure}

\subsection{Robustness and Invisibility Trade-off}\label{sec:Method_Theory}
As mentioned earlier Sec.~\ref{sec:background}, a complete watermarking framework can be divided into three components: encoder $\mathcal{E}$, decoder $\mathcal{D}$, and distortion layer $\mathcal{T}$. The decoder takes only a single watermarked image $I_{wm}$ as input. To achieve correct verification, the decoder must implicitly disentangle the image content from the embedded watermark information and correctly associate them to extract the watermark successfully.


\begin{definition}
An image and watermark information: $I$, $wm \ \subset \{0,1\}^k$, the encoder is:
$$\mathcal{E}(I, wm)=I+\epsilon \cdot \underbrace{\phi(I,wm)}_W$$
the decoder is:
$$\mathcal{D}(I_{wm}) \to \underbrace{(\hat{I}, \hat{W})}_{{match}} \to \hat{wm}$$

$\epsilon$ is the embedding strength.The feature space of the image $\mathcal{P} = \{p_1, p_2,...,p_n\}$ consists of two subspaces for embedding information: 
$$\mathcal{P} = \mathcal{P}_r \bigoplus \mathcal{P}_c$$

Due to joint training, the encoder exhibits a similar implicit decomposition behavior, projecting the input image $I$ into two feature spaces, named as $P_r$ and $P_c$. The former is a more suitable embedding space for information hiding, while the latter is not. 

The encoder performs this mapping $\mathcal{E}(I,wm) \to I_{wm}$ by:
$$\phi(I,wm) = \mathop{\min}_{p\in \mathcal{P}_r}||wm - \mathcal{D}(\mathcal{E}(p,wm))||^2+\lambda||\mathcal{E}(p,wm)||$$

However, as robustness requirements are introduced and continuously strengthened, the encoder must encode more information to ensure the watermark’s resistance to attacks. When the $P_r$   space is fully utilized, the encoder is forced to use $P_c$ for watermark embedding, polluting the $P_c$ space.

\end{definition}

\begin{definition}
An intuitive definition of embeddable threshold is:
\begin{gather*}
C(I) = \sup_{W \in \mathcal{P}_r}{\frac{||W||_2}{||I||_2}} \\\\
s.t. PNSR(I, I+W) \ge TV
\end{gather*}
$TV$ represents the lower bound of the visual quality.
\end{definition}

\begin{proposition}
When the robustness requirement exceeds $C(I)$, a decline in visual quality is inevitable.
\end{proposition}

\begin{proof}
Let the distortion layer $\mathcal{T}$ introduce noise $\eta \sim \mathcal{T}$, with the requirement that
$$||wm-\mathcal{D}(I_{wm} + \eta)|| \le \mathcal{B}$$
$\mathcal{B}$ is the bit error rate. Considering the channel capacity as:
$$R=\frac{1}{2}\log(1+\frac{\epsilon^2||W||^2}{\delta_{\eta}^2})$$
To achieve $R\ge H(wm)$, the following conditions must be met:
$$
\epsilon||W|| \le \sqrt{(2^{2H(wm)}-1)\delta_{\eta^2}}
$$
$H(wm)$ represents the entropy of $wm$. 

When $\sqrt{(2^{2H(wm)}-1)\delta_{\eta^2}} > C(I)||I||_2$, the system cannot simultaneously satisfy both, and it is necessary to increase $C(I)$, introducing visual artifacts into the image. Detailed proof is provided in Appendix~\ref{sec:Appendix_Proofs}.
\end{proof}
The artifacts introduced by sacrificing invisibility contain watermark information, creating a security vulnerability where watermark information leakage occurs.


\subsection{Detection Evasion}\label{sec:Method_Evasion Attack}
Our method is illustrated in Figure~\ref{fig:method-overview}, Suppose we have an image $I_{wm}$, embedded with an unknown watermark $wm$. This image is fed into a feature extraction module $\mathcal{F}(\cdot)$, resulting in multi-channel features $\mathcal{F}(I_{wm})$. To automate the selection of features that capture potential information leakage, we perform clustering on the multi-channel features. Among the resulting clusters, we identify the two clusters with the smallest number of samples and extract their corresponding feature channel positions $\mathcal{W}$.

To achieve the goal of an evasion attack, we need to disrupt the leaked watermark information captured from $I_{wm}$.We formulate this process as an optimization problem: finding a perturbation $\delta$ that disrupts the leaked information while preserving the visual quality of the image. The formulation is as follows:
\begin{equation}
\label{eq:1}
\begin{split}
    \mathop{\min}_{\delta}-\mathcal{L}(\mathcal{W} \cdot \mathcal{F}(I_{wm}), \mathcal{W}\cdot \mathcal{F}(I_{wm} + \delta)) \\
    \mathrm{ s.t.} ||\delta||_{\infty} < \epsilon
\end{split}
\end{equation}

where $\mathcal{L}(\cdot,\cdot)$ is the loss function that measures the distance between two features, and $\epsilon$ is a perturbation budget.

We use Projected Gradient Descent (PGD)~\cite{PGD} to solve the optimization problem in Eq~\ref{eq:1}. Finally, we complete the attack through $I_{wm} + \delta$.

Our detailed algorithm is shown as 
 Algorithm~\ref{alg:evasion algo}.


\subsection{Forgery Attack}
As shown in Figure~\ref{fig:method-overview}, we first use the feature extraction module and clustering algorithm to extract features containing leaked watermark information, from $I_{wm}$. To achieve the goal of spoofing, we still need to extract the leaked information. Therefore, this process can be formulated as the following optimization problem:
\begin{equation}
\label{eq:2}
\begin{split}
     \mathop{\min}_{\delta}-\mathcal{L}(\mathcal{W} \cdot \mathcal{F}(I_{wm}), \mathcal{W}\cdot \mathcal{F}(I_{wm} + \delta)) \\
    \mathrm{ s.t.} ||\delta||_{\infty} < \epsilon
\end{split}
\end{equation}
\vspace{-4mm}

where $\epsilon$ is a perturbation budget, and 
 this process is identical to the above evasion attack, referred to as Stage \uppercase\expandafter{\romannumeral1}.
However, the learned $\delta$ alone cannot fulfill the forgery purpose for \emph{semantic watermarking}. Based on the theory discussed earlier (See Sec.~\ref{sec:Method_Theory}), we need to consider the coupling effect between the semantics and watermark. After the optimization in Eq~\ref{eq:2} is completed, an additional optimization term should be included to further find another perturbation, $\delta_s$, which can be described as:
\begin{equation}
\label{eq:3}
\begin{split}
     \mathop{\min}_{\delta}\mathcal{L}((1-\mathcal{W}) \cdot \mathcal{F}(I_{wm}+\delta), (1-\mathcal{W})\cdot \mathcal{F}(I' + \delta_s)) \\
    \mathrm{ s.t.} ||\delta_s||_{\infty} < \epsilon
\end{split}
\end{equation}
This process is referred to as Stage \uppercase\expandafter{\romannumeral2}.
We use Projected Gradient Descent (PGD)~\cite{PGD} to solve the optimization problem in Eq~\ref{eq:2} and Eq~\ref{eq:3}.
Finally, we complete the attack through $\{I' - \delta\}$ or $\{I' - \delta + \delta_s \}$.

Our detailed algorithm is shown as Algorithm~\ref{alg:spoof algo}
\section{Evaluation}
\subsection{Setup}
\textbf{Dataset}. We use two datasets to validate our attack method, including COCO~\cite{MS-COCO} and DIV2K~\cite{Agustsson_2017_CVPR_Workshops}. We randomly selected 100 images from each dataset, and each image was embedded with random watermark information using different watermarking algorithms. Another 100 non-overlapping images were also selected from the COCO~\cite{MS-COCO} dataset to serve as clean images for the forgery attack.

\textbf{Watermark setting}. To evaluate our attack method, we test seven publicly available watermarking methods: DwtDctSvd~\cite{DWT-DCT-SVD}, DwtDct~\cite{DWT-DCT}, RivaGAN~\cite{RivaGAN}, StegaStamp~\cite{stegastamp}, HiDDeN~\cite{zhu2018hidden}, PIMoG~\cite{PIMoG}, and CIN~\cite{CIN}. 


\textbf{Attack Benchmarking}. 
For evasion attacks, we compared our approach not only with traditional image degradation techniques such as JPEG compression, Gaussian noise, and Gaussian blur but also with related attack methods, including WmAttacker~\cite{zhao2023invisible} and WmRobust~\cite{saberi2023robustness}. For forgery attacks, we compare our method with existing forgery techniques, including CopyAttack~\cite{kutter2000watermark}, Steganalysis~\cite{yang2024steganalysis}, and WmRobust~\cite{saberi2023robustness}.

\textbf{Evaluation metrics}. We evaluate the visual quality of the attacked watermarked images and the corresponding original watermarked images using two commonly used metrics: Structural Similarity Index Measure (SSIM)~\cite{wang2004image} and Peak Signal-to-Noise Ratio (PSNR). To evaluate the effectiveness of our attack method, we use bit accuracy and success rate (SR) as metrics. Bit accuracy refers to the correct matching rate between the watermark extracted from the image and its ground truth. SR represents the proportion of successfully attacked samples to the total number of samples. This varies under two attack scenarios: in an evasion attack, the attack is considered successful if the bit accuracy of the attacked image falls below a certain threshold. For a forgery attack, the opposite holds true.

\textbf{Parameter settings}. We use AdamW~\cite{loshchilov2017decoupled} as the optimizer, with a learning rate of 0.15 and a weight decay of 1e-4. We chose SSIM loss and L1 loss as the loss functions $\mathcal{L}(\cdot,\cdot)$ and employ a pre-trained DenseNet~\cite{DenseNet2017} as the feature extraction module $\mathcal{F}(\cdot)$. Regarding the detection thresholds, we uniformly set them to \{0.95, 0.9, 0.85, 0.8, 0.75, 0.7, 0.65, 0.6, 0.55\}, and also establish a threshold calculated based on watermark length $n$: A watermark will be detected if we can reject the null hypothesis $H_0$ with a p-value less than 0.05. The null hypothesis $H_0$ states that k matching bits were extracted from the watermarked image by random chance. This event has a probability of $P(X\ge k|H_0) = \sum_{i=k}^n{n\choose k}0.5^n$.

\begin{table*}[!th]
\centering
\caption{Overall evasion performance on the dataset COCO. Detailed results of full threshold candidates are shown in Appendix~\ref{sec:appendix:Evasion Attack against Related Works}.}
\label{tab:evade_coco_overall}
\resizebox{\textwidth}{!}{
\begin{tabular}{@{}cccc|ccc|ccc|ccc|ccc|ccc|ccc@{}}
\toprule
\multirow{2}{*}{\textbf{Methods}}   &\multicolumn{3}{c}{\textbf{PIMoG(th=0.6)}}  &\multicolumn{3}{c}{\textbf{HiDDeN(th=0.6)}}
&\multicolumn{3}{c}{\textbf{StegaStamp(th=0.57)}}
&\multicolumn{3}{c}{\textbf{DwtDct(th=0.625)}}
&\multicolumn{3}{c}{\textbf{DwtDctSvd(th=0.625)}}
&\multicolumn{3}{c}{\textbf{RivaGan(th=0.625)}}
&\multicolumn{3}{c}{\textbf{CIN(th=0.6)}}\\ \cmidrule(l){2-22} 
             & \textbf{SR$\uparrow$} & \textbf{SSIM$\uparrow$} & \textbf{PSNR$\uparrow$} & \textbf{SR$\uparrow$} & \textbf{SSIM$\uparrow$} & \textbf{PSNR$\uparrow$} & \textbf{SR$\uparrow$} & \textbf{SSIM$\uparrow$} & \textbf{PSNR$\uparrow$} & \textbf{SR$\uparrow$} & \textbf{SSIM$\uparrow$} & \textbf{PSNR$\uparrow$} & \textbf{SR$\uparrow$} & \textbf{SSIM$\uparrow$} & \textbf{PSNR$\uparrow$} & \textbf{SR$\uparrow$} & \textbf{SSIM$\uparrow$} & \textbf{PSNR$\uparrow$} & \textbf{SR$\uparrow$} & \textbf{SSIM$\uparrow$} & \textbf{PSNR$\uparrow$}  \\ \midrule
WmAttacker & 0 & 0.749 & 28.503 & 0.32 & 0.628 & 25.026 & 0 & 0.715 & 27.586 & 0.91 & 0.598 & 24.744 & 0.41 & 0.57 & 23.498 & 0.13 & 0.601 & 24.555 & 0.01 & 0.614 & 28.062 \\
WmRobust & 0.02 & \textbf{0.948} & \textbf{40.195} & 0.34 & 0.808 & 35.363 & 0.19 & \textbf{0.899} & \textbf{38.003} & 0.93 & 0.815 & \textbf{34.325} & 0.6 & 0.822 & \textbf{34.034} & 0.47 & 0.842 & 35.648 & 0.38 & 0.891 & 38.933\\
JPEG & 0 & 0.897 & 37.538 & 0.57 & 0.813 & 34.161 & 0 & 0.867 & 35.976 & 0.93 & 0.785 & 32.722 & 0.87 & 0.794 & 32.512 & 0.15 & 0.793 & 33.289 & 0.02 & \textbf{0.898} & \textbf{39.238}\\
Gaussian & 0.01 & 0.197 & 26.532 & \textbf{0.81} & 0.406 & 26.441 & 0 & 0.365 & 26.496 & \textbf{0.96} & 0.422 & 26.464 & 0.45 & 0.393 & 26.444 & 0.01 & 0.157 & 26.579 & 0 & 0.324 & 26.447\\
GaussianBlur & 0 & 0.771 & 31.362 & 0.05 & 0.651 & 30.567 & 0 & 0.713 & 30.123 & 0.85 & 0.603 & 28.435 & 0 & 0.61 & 28.378 & 0 & 0.615 & 28.433 & 0 & 0.761 & 31.495\\
\textbf{Ours} & \textbf{0.87} & 0.876 & 36.403 & 0.77 & \textbf{0.889} & \textbf{36.644} & \textbf{1} & 0.895 & 36.703 & \textbf{0.96} & \textbf{0.838} & 33.834 & \textbf{0.95} & \textbf{0.85} & 33.841 & \textbf{1} & \textbf{0.872} & \textbf{35.851} & \textbf{0.78} & 0.809 & 34.801\\
\bottomrule
\end{tabular}}
\vspace{-3mm}
\end{table*}

\begin{table*}[!th]
\centering
\caption{Overall evasion performance on the dataset DIV2K. Detailed results of full threshold candidates are shown in Appendix~\ref{sec:appendix:Evasion Attack against Related Works}.}
\label{tab:evade_div2k_overall}
\resizebox{\textwidth}{!}{
\begin{tabular}{@{}cccc|ccc|ccc|ccc|ccc|ccc|ccc@{}}
\toprule
\multirow{2}{*}{\textbf{Methods}}   &\multicolumn{3}{c}{\textbf{PIMoG(th=0.6)}}  &\multicolumn{3}{c}{\textbf{HiDDeN(th=0.6)}}
&\multicolumn{3}{c}{\textbf{StegaStamp(th=0.57)}}
&\multicolumn{3}{c}{\textbf{DwtDct(th=0.625)}}
&\multicolumn{3}{c}{\textbf{DwtDctSvd(th=0.625)}}
&\multicolumn{3}{c}{\textbf{RivaGan(th=0.625)}}
&\multicolumn{3}{c}{\textbf{CIN(th=0.6)}}\\ \cmidrule(l){2-22} 
             & \textbf{SR$\uparrow$} & \textbf{SSIM$\uparrow$} & \textbf{PSNR$\uparrow$} & \textbf{SR$\uparrow$} & \textbf{SSIM$\uparrow$} & \textbf{PSNR$\uparrow$} & \textbf{SR$\uparrow$} & \textbf{SSIM$\uparrow$} & \textbf{PSNR$\uparrow$} & \textbf{SR$\uparrow$} & \textbf{SSIM$\uparrow$} & \textbf{PSNR$\uparrow$} & \textbf{SR$\uparrow$} & \textbf{SSIM$\uparrow$} & \textbf{PSNR$\uparrow$} & \textbf{SR$\uparrow$} & \textbf{SSIM$\uparrow$} & \textbf{PSNR$\uparrow$} & \textbf{SR$\uparrow$} & \textbf{SSIM$\uparrow$} & \textbf{PSNR$\uparrow$}  \\ \midrule
WmAttacker & 0 & 0.739 & 26.853 & 0.34 & 0.64 & 23.954 & 0 & 0.71 & 26.2 & 0.89 & 0.553 & 21.427 & 0.38 & 0.527 & 20.62 & 0.15 & 0.536 & 19.978 & 0.01 & 0.644 & 23.031\\
WmRobust & 0.01 & \textbf{0.942} & \textbf{40.128} & 0.2 & 0.814 & \textbf{34.525} & 0.16 & \textbf{0.912} & \textbf{37.374} & 0.89 & 0.821 & 31.903 & 0.48 & 0.827 & 31.676 & 0.35 & \textbf{0.848} & \textbf{32.769} & 0.18 & \textbf{0.891} & \textbf{38.295}\\
JPEG & 0 & 0.897 & 36.37 & 0.42 & 0.822 & 33.522 & 0 & 0.867 & 35.976 & 0.89 & 0.766 & 29.993 & 0.77 & 0.774 & 29.897 & 0.12 & 0.771 & 29.76 & 0.01 & 0.85 & 35.727\\
Gaussian & 0 & 0.399 & 26.477 & 0.72 & 0.444 & 26.509 & 0 & 0.365 & 26.496 & 0.9 & 0.516 & 26.522 & 0.44 & 0.502 & 26.561 & 0.01 & 0.578 & 26.954 & 0 & 0.362 & 26.479\\
GaussianBlur & 0 & 0.725 & 30.154 & 0.05 & 0.62 & 27.565 & 0 & 0.713 & 30.123 & \textbf{0.91} & 0.505 & 25.97 & 0.03 & 0.368 & 22.169 & 0.01 & 0.526 & 24.617 & 0 & 0.735 & 30.952\\
\textbf{Ours} & \textbf{1} & 0.828 & 33.042 & \textbf{0.96} & \textbf{0.841} & 33.257 & \textbf{1} & 0.853 & 33.54 & 0.89 & \textbf{0.855} & \textbf{32.841} & \textbf{0.85} & \textbf{0.859} & \textbf{32.745}
 & \textbf{1} & 0.841 & 32.656 & \textbf{0.45} & 0.811 & 34.744\\
\bottomrule
\end{tabular}}
\vspace{-3mm}
\end{table*}

\begin{table*}[!th]
\centering
\caption{Overall spoofing performance on the dataset COCO. Detailed results of full threshold candidates are shown in Appendix~\ref{sec:Appendix_Additional Experimental Results_Spoof Attack}.}
\setlength{\tabcolsep}{13pt}
\label{tab:spoof_coco_overall}
\resizebox{\textwidth}{!}{
\begin{tabular}{@{}cccc|ccc|ccc|ccc|ccc@{}}
\toprule
\multirow{2}{*}{\textbf{Methods}}   &\multicolumn{3}{c}{\textbf{PIMoG(th=0.6)}}  &\multicolumn{3}{c}{\textbf{HiDDeN(th=0.6)}}
&\multicolumn{3}{c}{\textbf{StegaStamp(th=0.57)}}
&\multicolumn{3}{c}{\textbf{RivaGan(th=0.625)}}
&\multicolumn{3}{c}{\textbf{CIN(th=0.6)}}\\ \cmidrule(l){2-16} 
             & \textbf{SR$\uparrow$} & \textbf{SSIM$\uparrow$} & \textbf{PSNR$\uparrow$} & \textbf{SR$\uparrow$} & \textbf{SSIM$\uparrow$} & \textbf{PSNR$\uparrow$} & \textbf{SR$\uparrow$} & \textbf{SSIM$\uparrow$} & \textbf{PSNR$\uparrow$} & \textbf{SR$\uparrow$} & \textbf{SSIM$\uparrow$} & \textbf{PSNR$\uparrow$} & \textbf{SR$\uparrow$} & \textbf{SSIM$\uparrow$} & \textbf{PSNR$\uparrow$}  \\ \midrule
CopyAttack & 0.13 & 0.729 & 20.786 & 0.16 & 0.78 & 21.682 & 0.09 & 0.732 & 20.296 & 0.06 & 0.704 & 20.471 & 0.1 & 0.724 & 19.795\\
Steganalysis & 0.1 & 0.907 & \textbf{34.524} & 0.05 & \textbf{0.902} & \textbf{34.451} & 0.08 & \textbf{0.923} & \textbf{34.42} & 0.07 & \textbf{0.933} & \textbf{34.483} & 0.18 & \textbf{0.919} & \textbf{34.572}\\
WmRobust & 0.86 & \textbf{0.915} & 31.129 & 0.53 & 0.726 & 26.387 & 0.92 & 0.832 & 29.731 & 0.08 & 0.795 & 31.841 & \textbf{1} & 0.82 & 28.529\\
\textbf{Ours} & \textbf{1} & 0.809 & 33.485 & \textbf{0.91} & 0.704 & 33.137 & \textbf{1} & 0.827 & 34.01 & \textbf{0.18} & 0.683 & 33.312 & \textbf{1} & 0.807 & 33.716\\
\bottomrule
\end{tabular}}
\vspace{-3mm}
\end{table*}

\begin{figure}[!t]
    \centering
    \includegraphics[width=\linewidth]{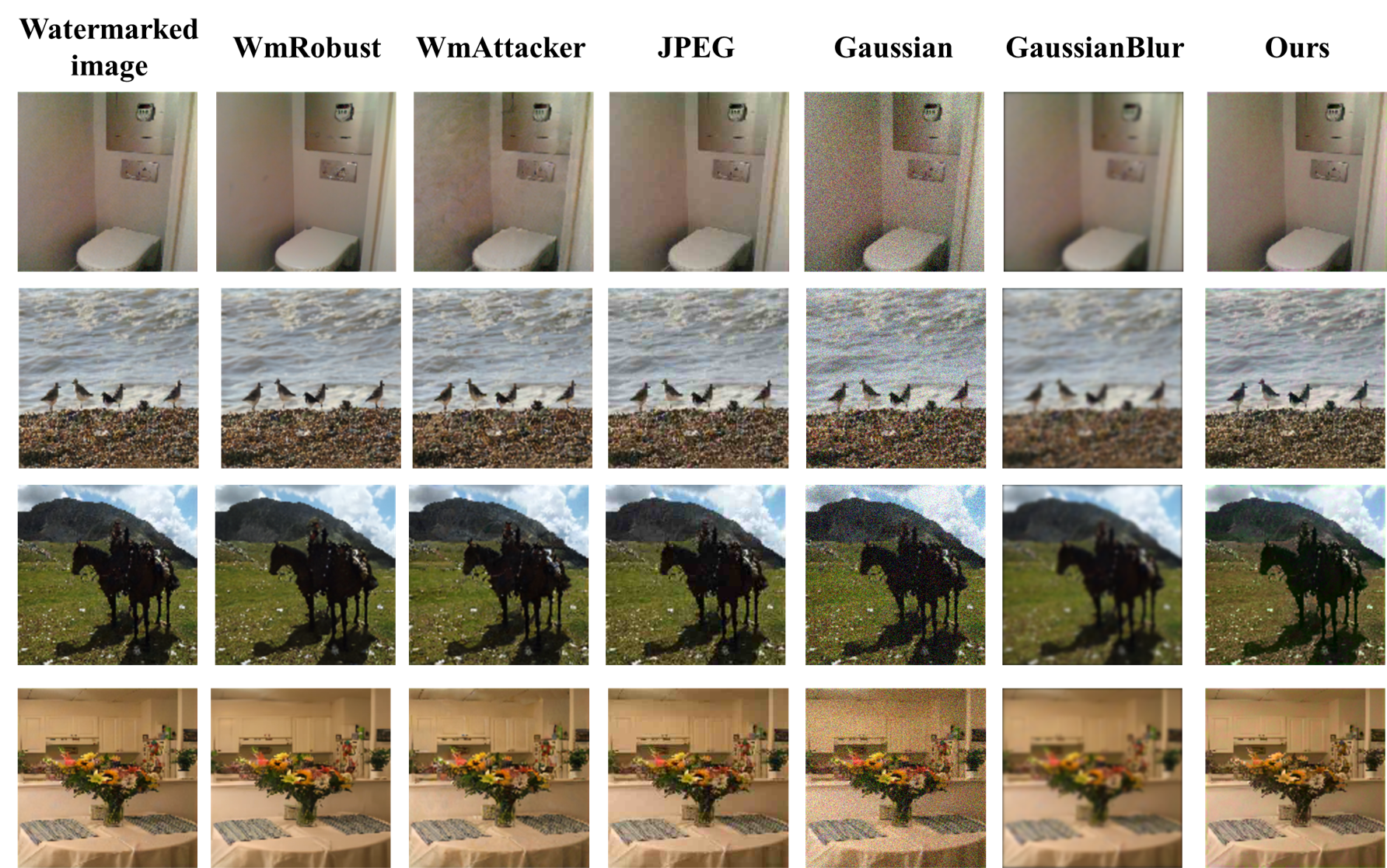} 
    
    \caption{Examples of watermark removal via our evasion attack on PIMoG}
    \label{fig:evasion-pimog}
    \vspace{-3mm}
\end{figure}

\begin{figure}[!t]
    \centering
    \includegraphics[width=\linewidth]{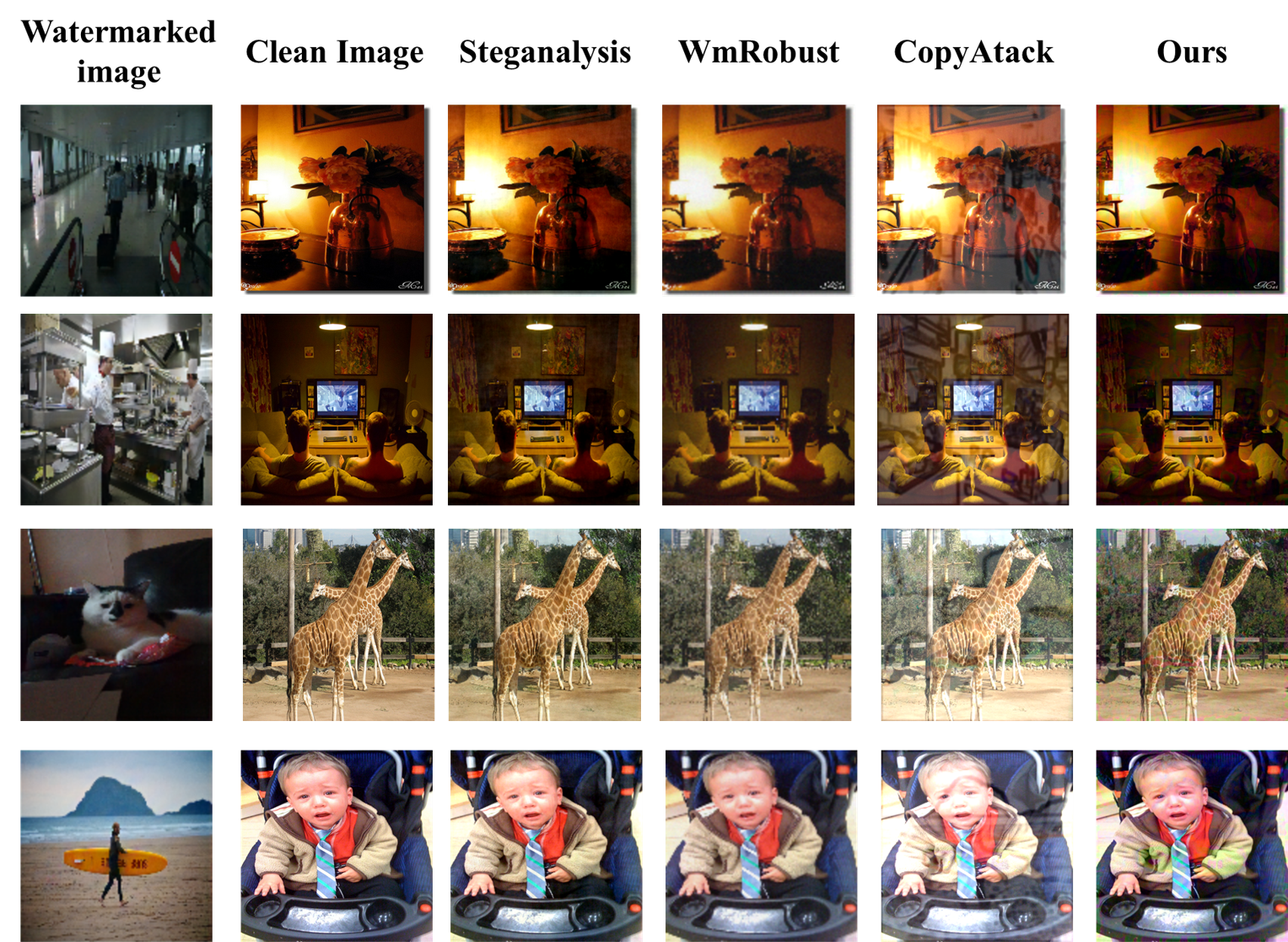} 
    \vspace{-3mm}
    \caption{Examples of watermark spoofing via our forgery attack on PIMoG.}
    \label{fig:spoof-pimog}
    \vspace{-3mm}
\end{figure}

\subsection{Results and Analysis}
In this section, we present the detailed results of our attacks and provide an analysis of the relevant findings. More detailed experimental results can be found in Appendix ~\ref{sec:Appendix_Additional Experimental Results}.

\textbf{Evasion Attack}. Tables~\ref{tab:evade_coco_overall} and Table~\ref{tab:evade_div2k_overall} summarize the comparison of our method, related methods, and image degradation factors regarding success rate and visual fidelity metrics. Our method achieves optimal or suboptimal attack success rates across different datasets and algorithms while maintaining high visual quality. Overall, our method achieves an average success rate improvement of 60\% on the COCO dataset and 61\% on the DIV2K dataset compared to other methods. Figure~\ref{fig:evasion-pimog} shows some examples of the evasion attack.

\textbf{Forgery Attack}. Table~\ref{tab:spoof_coco_overall} presents a performance comparison between our method and related forgery methods. Our method achieves the highest forgery success rate while preserving visual fidelity. Compared to related forgery methods, our method achieves an average success rate improvement of 51\%. We found that WmRobust~\cite{saberi2023robustness} also performs well when targeting certain watermarking models. However, WmRobust requires access to the target watermarking algorithm's encoder, whereas our method does not. Figure~\ref{fig:spoof-pimog} shows some examples of the forgery attack.

\subsection{Ablation Study}
\textbf{Ablation study for evasion attack}. Table~\ref{tab:evade_coco_ablation} reports the results of our evasion attack and the results of ablating different parts of the method, validating the effectiveness of our design. The experiments are divided into two parts: feature extraction network and feature channel position retrieval. \textbf{w/o $\mathcal{F}$} indicates the use of the original image without feature extraction, while \textbf{w/o $\mathcal{W}$} indicates the use of all channels without selective retrieval. 

Using only the feature extraction module $\mathcal{F}$ without channel position retrieval achieves strong attacks but severely degrades image quality. Optimizing directly on the original image worsens visual perception and reduces attack effectiveness. These results confirm that watermark leakage can be captured via feature extraction, and precise localization can minimize image distortion.

\textbf{Ablation study for forgery attack}. Table~\ref{tab:spoof_coco_ablation} presents the ablation results of our forgery attack, divided into three parts: Stage \uppercase\expandafter{\romannumeral1}, Only Stage \uppercase\expandafter{\romannumeral2} and Stage \uppercase\expandafter{\romannumeral1} +  Stage \uppercase\expandafter{\romannumeral2}. Note that Only Stage \uppercase\expandafter{\romannumeral2} is shown as follows:

\begin{equation}
\label{eq:process-2}
\begin{split}
     \mathop{\min}_{\delta_s}\mathcal{L}((1-\mathcal{W}) \cdot \mathcal{F}(I_{wm}), (1-\mathcal{W})\cdot \mathcal{F}(I' + \delta_s)) \\
    \mathrm{ s.t.} ||\delta_s||_{\infty} < \epsilon
\end{split}
\end{equation}

Experimental results show that we successfully forge watermarks of three algorithms—PIMoG, StegaStamp, and CIN, achieving optimal performance using Stage \uppercase\expandafter{\romannumeral1}, where the attack merely extracts leaked watermark information and embeds it into a clean image. This indicates that these schemes do not enforce a strong alignment between watermark information and image content, making them highly vulnerable to forgery attacks.

However, for RivaGan and HiDDeN, the Stage \uppercase\expandafter{\romannumeral1} forgery attack performs poorly, while combining Stage \uppercase\expandafter{\romannumeral1} and Stage \uppercase\expandafter{\romannumeral2} obviously improves the forgery effect, suggesting that RivaGan and HiDDeN facilitate a more substantial alignment between image semantics and watermark information.

Notably, all algorithms achieve satisfactory forgery results through Only Stage \uppercase\expandafter{\romannumeral2}. We conjecture two reasons for this: 1) Some leakage information of watermark can be identified from the less-watermark-related channels with Stage \uppercase\expandafter{\romannumeral2}; 2) Stage \uppercase\expandafter{\romannumeral2} considers more semantic information from watermarked images, improving the alignment between watermark information and image semantics. The results demonstrate the effectiveness of our method design and also support the theory proposed in Sec.~\ref{sec:Method_Theory}.

\begin{table*}[!th]
\centering
\caption{Ablation study for evasion attack. Detailed results of full threshold candidates are shown in Appendix~\ref{sec:appendix:Ablation_Study_for_Evasion_Attack}.}
\label{tab:evade_coco_ablation}
\resizebox{\textwidth}{!}{
\begin{tabular}{@{}cccc|ccc|ccc|ccc|ccc|ccc|ccc@{}}
\toprule
\multirow{2}{*}{\textbf{Methods}}   &\multicolumn{3}{c}{\textbf{PIMoG(th=0.6)}}  &\multicolumn{3}{c}{\textbf{HiDDeN(th=0.6)}}
&\multicolumn{3}{c}{\textbf{StegaStamp(th=0.57)}}
&\multicolumn{3}{c}{\textbf{DwtDct(th=0.625)}}
&\multicolumn{3}{c}{\textbf{DwtDctSvd(th=0.625)}}
&\multicolumn{3}{c}{\textbf{RivaGan(th=0.625)}}
&\multicolumn{3}{c}{\textbf{CIN(th=0.6)}}\\ \cmidrule(l){2-22} 
             & \textbf{SR$\uparrow$} & \textbf{SSIM$\uparrow$} & \textbf{PSNR$\uparrow$} & \textbf{SR$\uparrow$} & \textbf{SSIM$\uparrow$} & \textbf{PSNR$\uparrow$} & \textbf{SR$\uparrow$} & \textbf{SSIM$\uparrow$} & \textbf{PSNR$\uparrow$} & \textbf{SR$\uparrow$} & \textbf{SSIM$\uparrow$} & \textbf{PSNR$\uparrow$} & \textbf{SR$\uparrow$} & \textbf{SSIM$\uparrow$} & \textbf{PSNR$\uparrow$} & \textbf{SR$\uparrow$} & \textbf{SSIM$\uparrow$} & \textbf{PSNR$\uparrow$} & \textbf{SR$\uparrow$} & \textbf{SSIM$\uparrow$} & \textbf{PSNR$\uparrow$}  \\ \midrule
$\mathcal{W}$, $\mathcal{F}$ & 0.87 & \textbf{0.876} & \textbf{36.403} & 0.77 & \textbf{0.889} & \textbf{36.644} & \textbf{1} & \textbf{0.895} & \textbf{36.703}
& \textbf{0.96} & \textbf{0.838} & \textbf{33.834} & \textbf{0.95} & \textbf{0.85} & \textbf{33.841} & \textbf{1} & \textbf{0.872}	& \textbf{35.851} & 0.78	& \textbf{0.809}	& \textbf{34.801} \\
w/o $\mathcal{W}$, $\mathcal{F}$ & \textbf{1} & 0.331 & 28.591 & \textbf{0.86} & 0.324 & 27.93 & \textbf{1} & 0.345 & 28.41 & 0.91 & 0.386 & 27.873 & 0.93 & 0.382 & 27.873 & \textbf{1} & 0.373 & 27.829 & 0.89 & 0.314 & 28.789 \\
w/o $\mathcal{W}$, w/o $\mathcal{F}$ & 0 & 0.184 & 27.605 & 0.52 & 0.163 & 27.644 & 0 & 0.195 & 27.615 & 0.62 & 0.245 & 27.507 & 0.32 & 0.125 & 27.996 & 0.3 & 0.118 & 27.973 & \textbf{1} & 0.142 & 27.762\\
\bottomrule
\end{tabular}}
\vspace{-4mm}
\end{table*}

\begin{table*}[!th]
\centering
\caption{Ablation study for spoofing attack. Detailed results of full threshold candidates are shown in Appendix~\ref{sec:appendix:Ablation_Study_for_Spoof_Attack}.}
\setlength{\tabcolsep}{13pt}
\label{tab:spoof_coco_ablation}
\resizebox{\textwidth}{!}{
\begin{tabular}{@{}cccc|ccc|ccc|ccc|ccc@{}}
\toprule
\multirow{2}{*}{\textbf{Methods}}   &\multicolumn{3}{c}{\textbf{PIMoG(th=0.6)}}  &\multicolumn{3}{c}{\textbf{HiDDeN(th=0.6)}}
&\multicolumn{3}{c}{\textbf{StegaStamp(th=0.57)}}
&\multicolumn{3}{c}{\textbf{RivaGan(th=0.625)}}
&\multicolumn{3}{c}{\textbf{CIN(th=0.6)}}\\ \cmidrule(l){2-16} 
             & \textbf{SR$\uparrow$} & \textbf{SSIM$\uparrow$} & \textbf{PSNR$\uparrow$} & \textbf{SR$\uparrow$} & \textbf{SSIM$\uparrow$} & \textbf{PSNR$\uparrow$} & \textbf{SR$\uparrow$} & \textbf{SSIM$\uparrow$} & \textbf{PSNR$\uparrow$} & \textbf{SR$\uparrow$} & \textbf{SSIM$\uparrow$} & \textbf{PSNR$\uparrow$} & \textbf{SR$\uparrow$} & \textbf{SSIM$\uparrow$} & \textbf{PSNR$\uparrow$}  \\ \midrule
Stage \uppercase\expandafter{\romannumeral1} & \textbf{1} & \textbf{0.809} & \textbf{33.485} & 0.15 & 0.659 & 28.375 & \textbf{1} & \textbf{0.827} & \textbf{34.01} & 0.12 & \textbf{0.683} & 28.403 & \textbf{1} & \textbf{0.807} & \textbf{33.716}\\
Only Stage \uppercase\expandafter{\romannumeral2} & \textbf{1} & 0.702 & 32.194 & \textbf{0.93} & 0.668 & 31.981 & 0.98 & 0.701 & 32.172 & \textbf{0.26} & 0.674 & 32.174 & \textbf{1} & 0.701 & 32.151\\
Stage \uppercase\expandafter{\romannumeral1} + Stage \uppercase\expandafter{\romannumeral2} & \textbf{1} & 0.684 & 31.818 & 0.91 & \textbf{0.704} & \textbf{33.137} & 0.99 & 0.685 & 31.779 & 0.18 & \textbf{0.683} & \textbf{33.312} & \textbf{1} & 0.679 & 31.769
\\
\bottomrule
\end{tabular}}
\vspace{-3mm}
\end{table*}
\section{Related Work}
\subsection{Image Watermarking Methods}
\emph{Non-learning-based watermarking methods} have developed for decades. Invisible-watermark~\cite{invisible-watermark}, a representative method deployed by Stable Diffusion, encodes watermark into frequency sub-bands. \emph{Learning-based watermarking methods} are gaining dominance due to their superior performance. Zhu et al. \cite{zhu2018hidden} propose the first end-to-end learning architecture for robust watermarking. Following this trend, a series of studies continue to enhance robustness against real-world interferences~\cite{stegastamp, Liu2019TwoStage}. 




\subsection{Detection Evasion Attacks}
\textbf{Destruction and Reconstruction}. The watermarked image firstly undergoes a certain level of degradation, followed by reconstruction to obtain a purified image. The mainstream approach for this method involves adding noise to the image and then using generative models, such as Diffusion Models (DM)~\cite{ho2020denoising}, for reconstruction and generation~\cite{an2024benchmarking, saberi2023robustness, zhao2023invisible}. In contrast, UnMarker ~\cite{Kassis2024Unmarker} employs a learnable filter to process the watermarked image, supplemented by visual loss functions to ensure and enhance the visual quality of the attack results. 

\textbf{Adversarial Attacks}. Transferring classic adversarial attack methods to the watermarking domain primarily targets the decoder of watermark models. WEVADE~\cite{jiang2023evading} incorporates both black-box and white-box adversarial attack methods. Lukas et al.~\cite{lukas2024leveraging} employ a surrogate model closely resembling the target model to perform transfer attacks. WmRobust~\cite{saberi2023robustness} requires a dataset containing both watermarked and non-watermarked images to train a feature classifier subjected to adversarial attacks. The attacks are transferred to the target watermark detection module. Similarly, WAVES~\cite{an2024benchmarking} relies on a relevant watermark dataset for surrogate attacks but introduces a more detailed classification of watermark data. Hu et al.~\cite{hu2024transfer} explore the feasibility of large-scale ensemble surrogate models for transfer attacks against target watermark models.

\subsection{Watermark Forgery Attacks}
CopyAttack~\cite{kutter2000watermark} was the first to introduce the concept of spoof attacks and proposed a method for predicting watermarks in unknown watermarking algorithm scenarios, embedding them into other images to achieve forgery. WmRobust~\cite{saberi2023robustness} proposes an attack leveraging the encoder of the watermark model to embed noise with a watermark and applies fine-tuning to generate forged watermarks. Steganalysis~\cite{yang2024steganalysis} computes a residual by statistically analyzing a dataset of watermarked images and unpaired clean images. This residual is then used to facilitate watermark forgery.

\section{Discussion and Limitations}
\subsection{Effectiveness against In-processing Watermarks}

In-processing watermarks usually demonstrate a stronger coupling effect with the semantics of imagery, probably presenting insufficient pattern leakage for our method to utilize. For example, Tree-Ring~\cite{wen2023tree} subtly influences the entire image generation process by embedding a pattern structured in the Fourier domain into the noise vector for sampling. This watermarking significantly connects with the image content, and our attacks exhibit limited performance on it.

While in-processing watermarks present challenges for our method, this watermarking approach has a limited range of applications. It necessitates significant modifications to existing image synthesis algorithms and cannot be applied to watermarking real images.


\subsection{Defenses against DAPAO}
We urgently need to develop defenses for post-processing watermarks against our attacks, considering these methods have wider application scenarios. A promising direction is to optimize the adversarial training (described in Sec.~\ref{sec:background}) with the inclusion of leakage estimation, striking a balance between robustness and security. 



\section{Conclusion}
We reveal a tradeoff in robust watermarks: Improved redundancy of watermark information enhances robustness, but increased redundancy raises the risk of watermark leakage. We propose DAPAO attack, a framework that requires only one image for watermark extraction, effectively achieving both watermark removal and spoofing attacks against cutting-edge robust watermarking methods. Our attack reaches an average success rate of 87\% in detection evasion (about 60\% higher than existing evasion attacks) and an average success rate of 85\% in forgery (approximately 51\% higher than current forgery studies). 
\section*{Impact Statement}

Watermarking is an avenue for AIGC provenance and detection, preventing potential misbehavior such as the spread of misinformation, copyright violation, and adversarial false attribution. Our work primarily underscores novel threats to modern learning-based watermarking schemes prioritizing robustness against real-world distortions. In theory, attackers could exploit these vulnerabilities to compromise watermarks, potentially harming users and service providers. However, the watermarking methods analyzed in this study are all open-source and research-focused, while the real-world deployment of invisible and robust watermarks remains in its early stages. Therefore, we believe making our work public has no direct negative impact. Conversely, we believe our findings have a positive societal impact by exposing a fundamental vulnerability in existing robust watermarking techniques, thereby preventing potential covert exploitation by adversaries and offering valuable insights for developing more secure image watermarking solutions.


\nocite{langley00}

\bibliography{reference}
\bibliographystyle{dapao}

\newpage
\appendix
\onecolumn
\section{Additional Experimental Results}\label{sec:Appendix_Additional Experimental Results}
\subsection{Evasion Attack against Related Works}
\label{sec:appendix:Evasion Attack against Related Works}
We provide detailed evasion attack results as shown in ~\cref{fig:evasion-related-pimog-psnr} - ~\cref{fig:evasion-related-hidden-ssim} on COCO dataset and shown in ~\cref{fig:evasion-related-pimog-div2k-psnr}-~\cref{fig:evasion-related-cin-div2k-ssim} on DIV2K dataset, including the attack success rates and visual metrics corresponding to all watermark detection thresholds. Additionally, we provide examples of attack results against other algorithms, as shown in ~\cref{fig:evasion-hidden,fig:evasion-pimog}.

\begin{figure}[!t]
\begin{minipage}{0.48\linewidth}
    \centering
    \includegraphics[width=\linewidth]{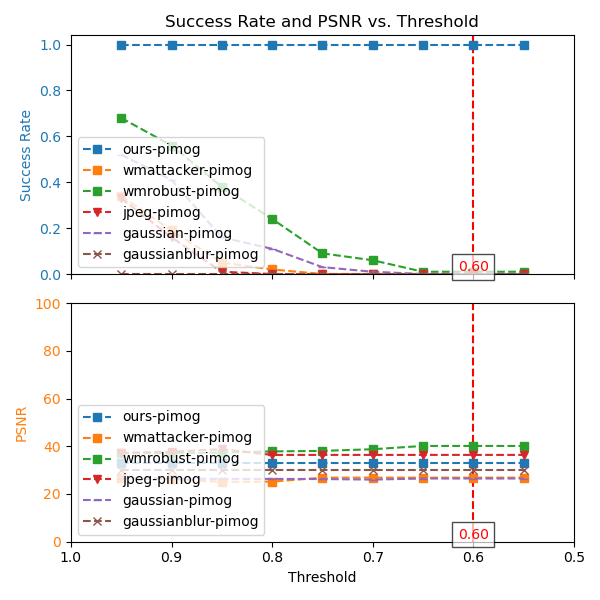} 
    \vspace{-6mm}
    \caption{The detailed success rate of PIMoG evasion attacks and the corresponding PSNR visual metric on DIV2K dataset.}
    \label{fig:evasion-related-pimog-div2k-psnr}
\end{minipage}\hfill
\begin{minipage}{0.48\linewidth}
    \centering
    \includegraphics[width=\linewidth]{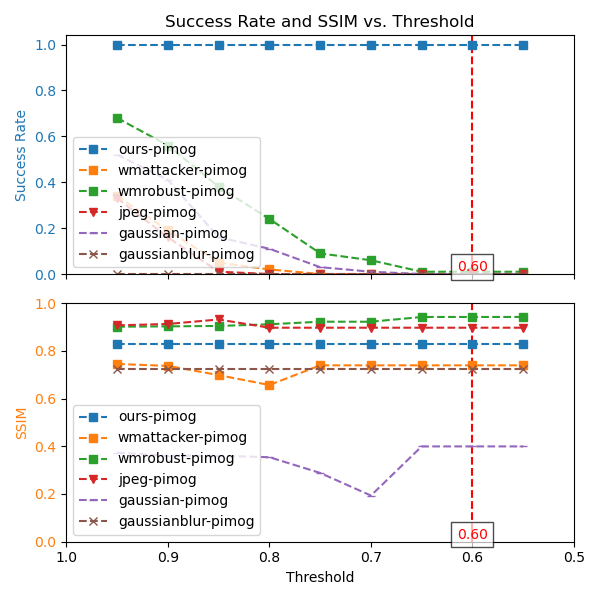}
    \vspace{-3mm}
    \caption{The detailed success rate of PIMoG evasion attacks and the corresponding SSIM visual metric on DIV2K dataset.}
    \label{fig:evasion-related-pimog-div2k-ssim}
\end{minipage}
\end{figure}

\begin{figure}[!t]
\begin{minipage}{0.48\linewidth}
    \centering
    \includegraphics[width=\linewidth]{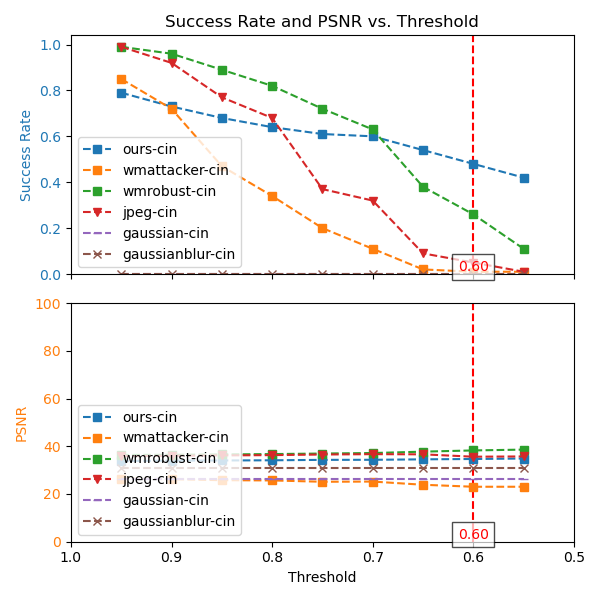} 
    \vspace{-6mm}
    \caption{The detailed success rate of CIN evasion attacks and the corresponding PSNR visual metric on DIV2K dataset.}
    \label{fig:evasion-related-cin-div2k-psnr}
\end{minipage}\hfill
\begin{minipage}{0.48\linewidth}
    \centering
    \includegraphics[width=\linewidth]{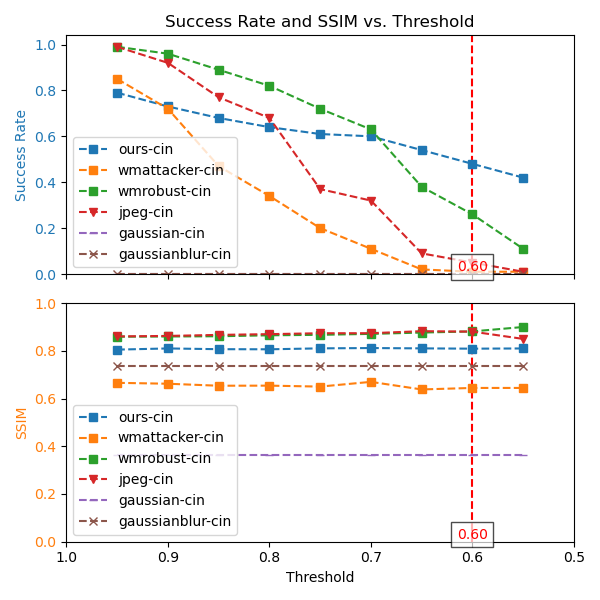}
    \vspace{-3mm}
    \caption{The detailed success rate of CIN evasion attacks and the corresponding SSIM visual metric on DIV2K dataset.}
    \label{fig:evasion-related-cin-div2k-ssim}
\end{minipage}
\end{figure}

\begin{figure}[!t]
\begin{minipage}{0.48\linewidth}
    \centering
    \includegraphics[width=\linewidth]{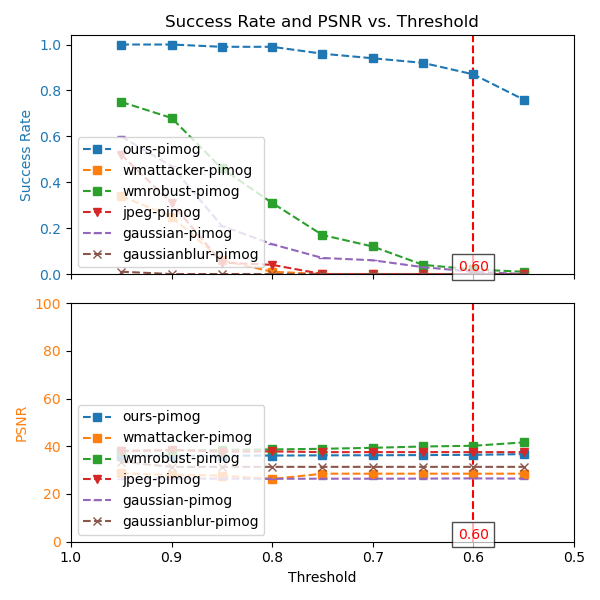} 
    \vspace{-6mm}
    \caption{The detailed success rate of PIMoG evasion attacks and the corresponding PSNR visual metric.}
    \label{fig:evasion-related-pimog-psnr}
\end{minipage}\hfill
\begin{minipage}{0.48\linewidth}
    \centering
    \includegraphics[width=\linewidth]{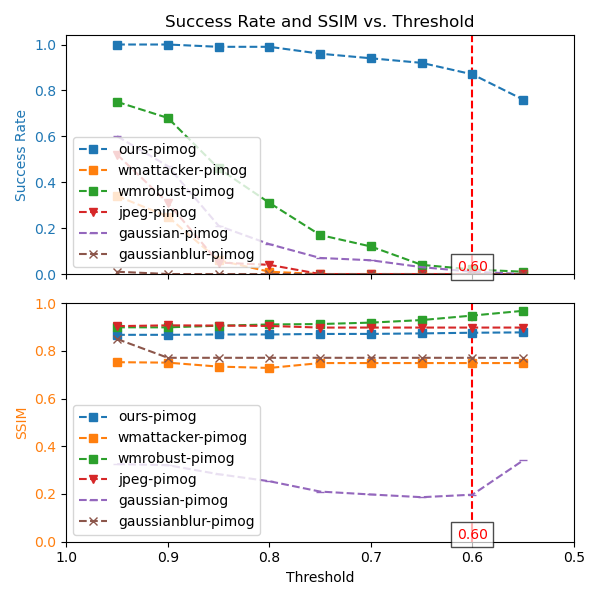} 
   
    \vspace{-3mm}
    \caption{The detailed success rate of PIMoG evasion attacks and the corresponding SSIM visual metric.}
    \label{fig:evasion-related-pimog-ssim}
\end{minipage}
\end{figure}

\begin{figure}[!t]
\begin{minipage}{0.48\linewidth}
    \centering
    \includegraphics[width=\linewidth]{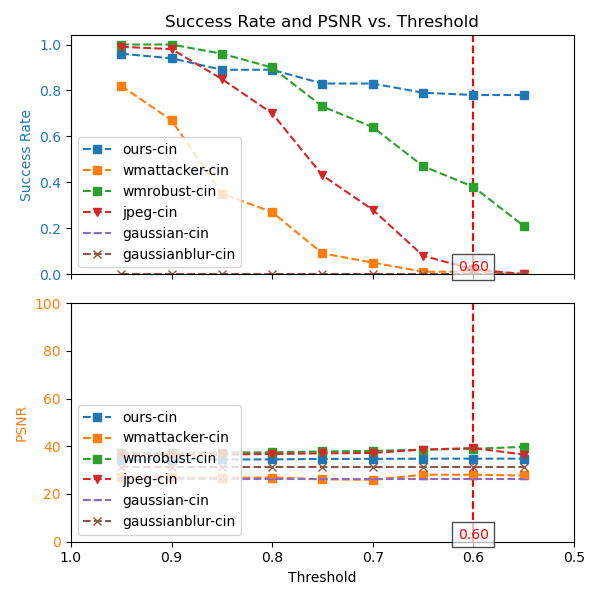} 
    
    \vspace{-6mm}
    \caption{The detailed success rate of CIN evasion attacks and the corresponding PSNR visual metric.}
    \label{fig:evasion-related-cin-psnr}
\end{minipage}\hfill
\begin{minipage}{0.48\linewidth}
    \centering
    \includegraphics[width=\linewidth]{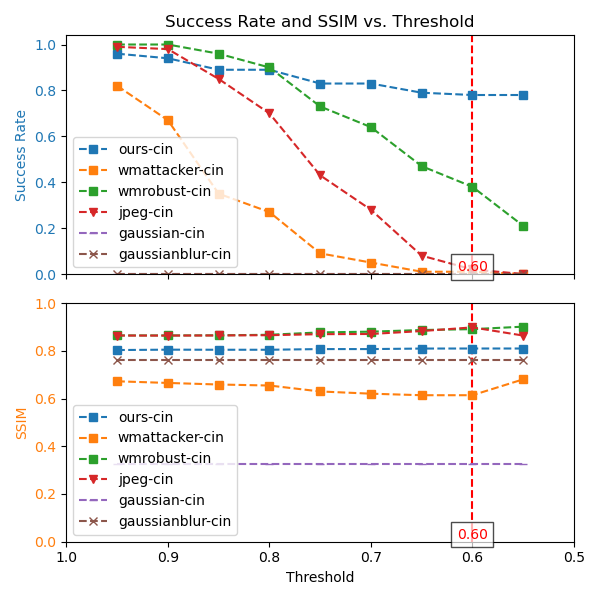} 
    
    \vspace{-6mm}
    \caption{The detailed success rate of CIN evasion attacks and the corresponding SSIM visual metric.}
    \label{fig:evasion-related-cin-ssim}
\end{minipage}
\end{figure}

\begin{figure}[!t]
\begin{minipage}{0.48\linewidth}
    \centering
    \includegraphics[width=\linewidth]{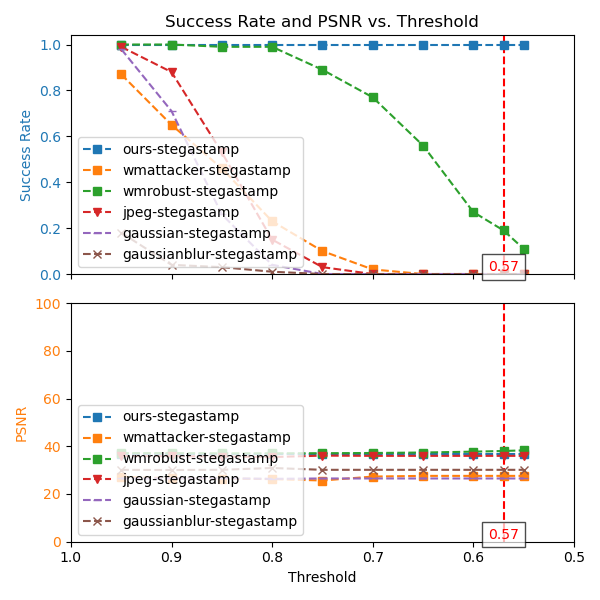} 
    
    \vspace{-6mm}
    \caption{The detailed success rate of StegaStamp evasion attacks and the corresponding PSNR visual metric.}
    \label{fig:evasion-related-stegastamp-psnr}
\end{minipage}\hfill
\begin{minipage}{0.48\linewidth}
    \centering
    \includegraphics[width=\linewidth]{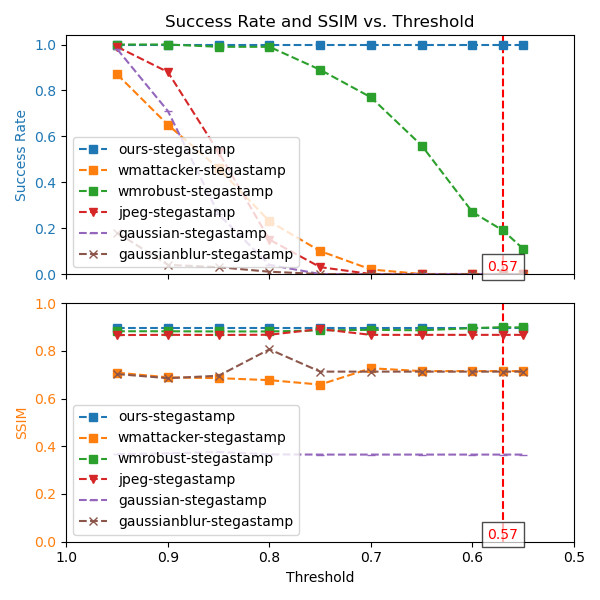} 
    
    \vspace{-6mm}
    \caption{The detailed success rate of StegaStamp evasion attacks and the corresponding SSIM visual metric.}
    \label{fig:evasion-related-stegastamp-ssim}
\end{minipage}
\end{figure}

\begin{figure}[!t]
\begin{minipage}{0.48\linewidth}
    \centering
    \includegraphics[width=\linewidth]{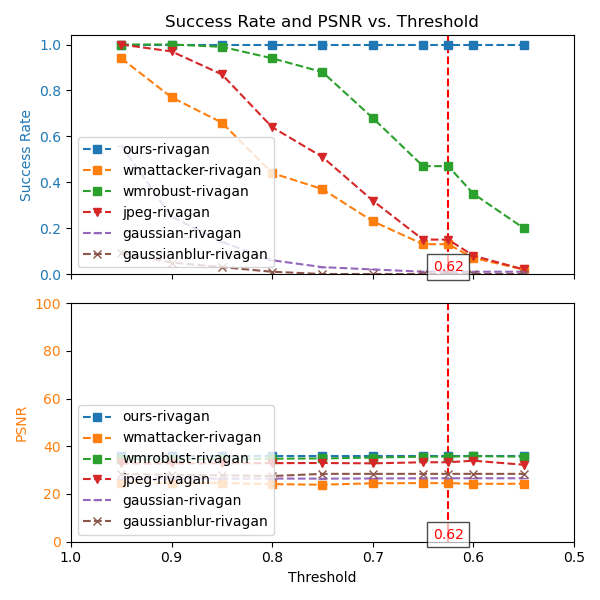} 
    
    \vspace{-6mm}
    \caption{The detailed success rate of RivaGan evasion attacks and the corresponding PSNR visual metric.}
    \label{fig:evasion-related-rivagan-psnr}
\end{minipage}\hfill
\begin{minipage}{0.48\linewidth}
    \centering
    \includegraphics[width=\linewidth]{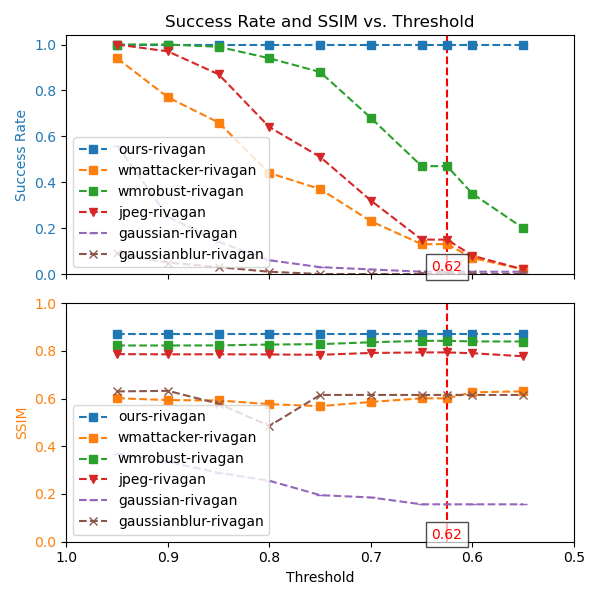} 
    
    \vspace{-6mm}
    \caption{The detailed success rate of RivaGan evasion attacks and the corresponding SSIM visual metric.}
    \label{fig:evasion-related-rivagan-ssim}
\end{minipage}
\end{figure}

\begin{figure}[!t]
\begin{minipage}{0.48\linewidth}
    \centering
    \includegraphics[width=\linewidth]{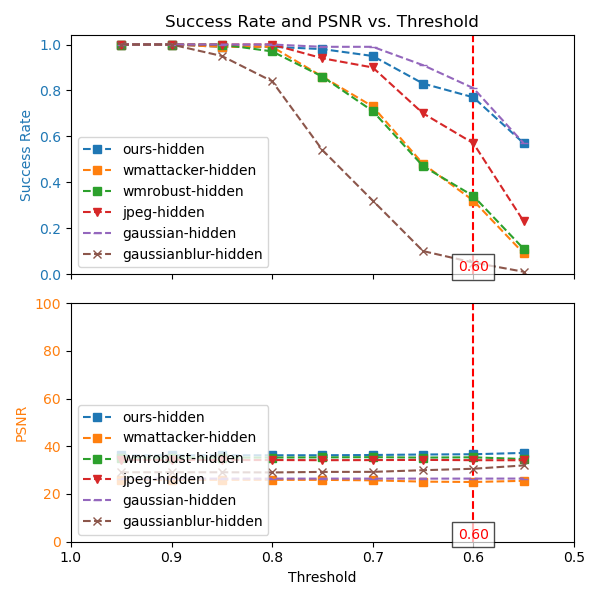} 
    
    \vspace{-6mm}
    \caption{The detailed success rate of HiDDeN evasion attacks and the corresponding PSNR visual metric.}
    \label{fig:evasion-related-hidden-psnr}
\end{minipage}\hfill
\begin{minipage}{0.48\linewidth}
    \centering
    \includegraphics[width=\linewidth]{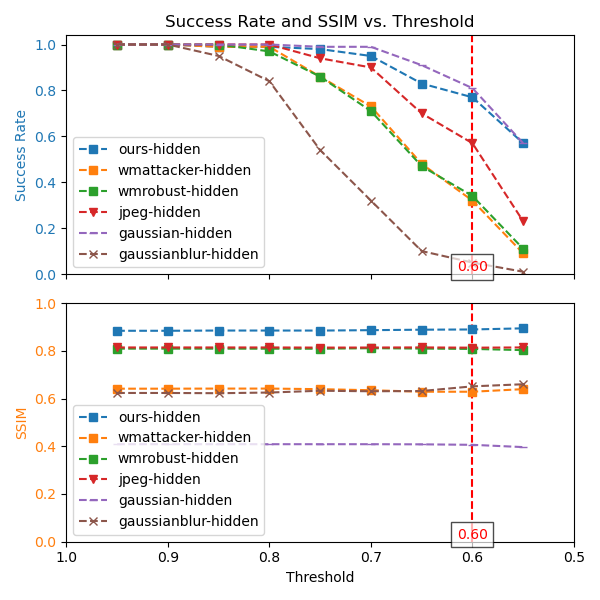} 
    
    \vspace{-6mm}
    \caption{The detailed success rate of HiDDeN evasion attacks and the corresponding SSIM visual metric.}
    \label{fig:evasion-related-hidden-ssim}
\end{minipage}
\end{figure}

\subsection{Spoof Attack against Related Works}\label{sec:Appendix_Additional Experimental Results_Spoof Attack}
We provide detailed spoof attack results as shown in ~\cref{fig:spoof-related-pimog-psnr}-~\cref{fig:spoof-related-hidden-ssim}, including the attack success rates and visual metrics corresponding to all watermark detection thresholds. Additionally, we provide examples of attack results against other algorithms, as shown in ~\cref{fig:spoof-pimog,fig:spoof-hidden}.

\subsection{Ablation Study for Evasion Attack}
\label{sec:appendix:Ablation_Study_for_Evasion_Attack}
We provide detailed ablation study for evasion attack results as shown in ~\cref{fig:evasion-ab-pimog-psnr}-~\cref{fig:evasion-ab-cin-ssim}, including the attack success rates and visual metrics corresponding to all watermark detection thresholds.

\begin{figure}[!t]
\begin{minipage}{0.48\linewidth}
    \centering
    \includegraphics[width=\linewidth]{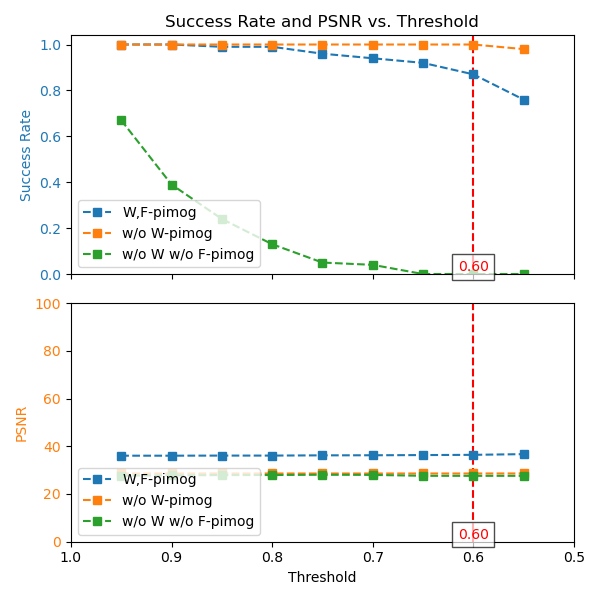} 
    \vspace{-6mm}
    \caption{The detailed success rate of PIMoG evasion attacks and the corresponding PSNR visual metric.}
    \label{fig:evasion-ab-pimog-psnr}
\end{minipage}\hfill
\begin{minipage}{0.48\linewidth}
    \centering
    \includegraphics[width=\linewidth]{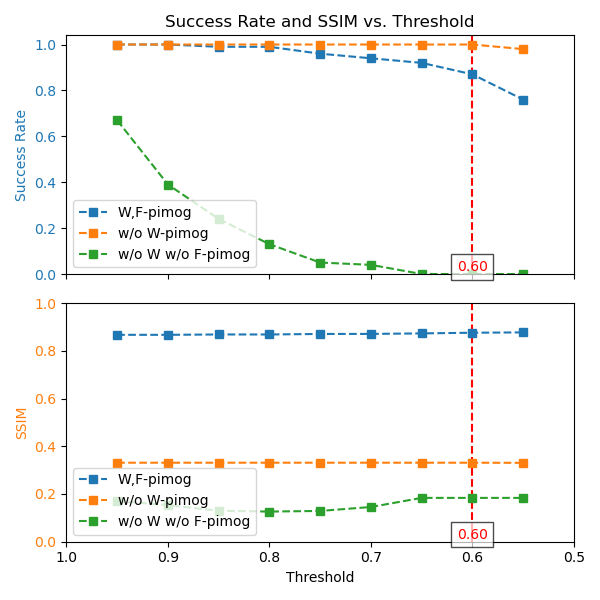} 
    \vspace{-6mm}
    \caption{The detailed success rate of PIMoG evasion attacks and the corresponding SSIM visual metric.}
    \label{fig:evasion-ab-pimog-ssim}
\end{minipage}
\end{figure}

\begin{figure}[!t]
\begin{minipage}{0.48\linewidth}
    \centering
    \includegraphics[width=\linewidth]{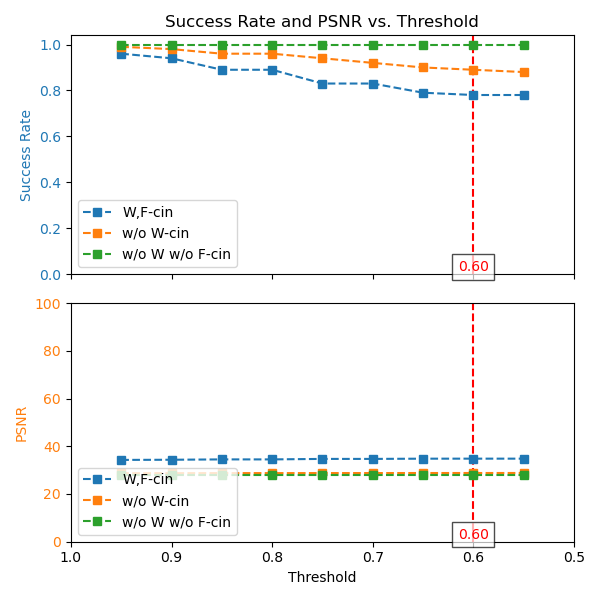} 
    \vspace{-6mm}
    \caption{The detailed success rate of CIN evasion attacks and the corresponding PSNR visual metric.}
    \label{fig:evasion-ab-cin-psnr}
\end{minipage}\hfill
\begin{minipage}{0.48\linewidth}
    \centering
    \includegraphics[width=\linewidth]{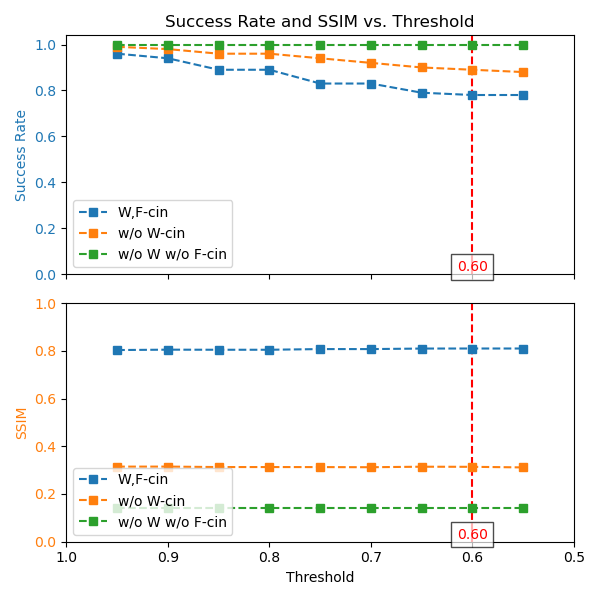} 
    \vspace{-6mm}
    \caption{The detailed success rate of CIN evasion attacks and the corresponding SSIM visual metric.}
    \label{fig:evasion-ab-cin-ssim}
\end{minipage}
\end{figure}

\subsection{Ablation Study for Spoof Attack}
\label{sec:appendix:Ablation_Study_for_Spoof_Attack}
We provide detailed ablation study for spoofing attack results as shown in ~\cref{fig:spoof-ab-pimog-psnr}-~\cref{fig:spoof-ab-cin-ssim}, including the attack success rates and visual metrics corresponding to all watermark detection thresholds.

\begin{figure}[!t]
\begin{minipage}{0.48\linewidth}
    \centering
    \includegraphics[width=\linewidth]{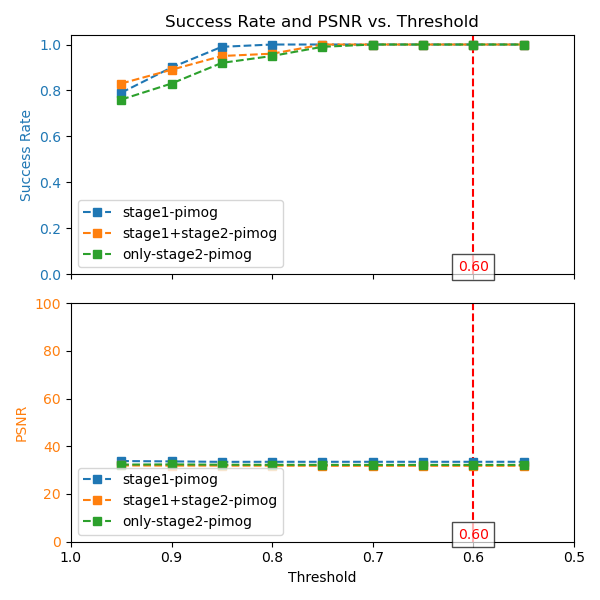} 
    \vspace{-6mm}
    \caption{The detailed success rate of PIMoG forgery attacks and the corresponding PSNR visual metric.}
    \label{fig:spoof-ab-pimog-psnr}
\end{minipage}\hfill
\begin{minipage}{0.48\linewidth}
    \centering
    \includegraphics[width=\linewidth]{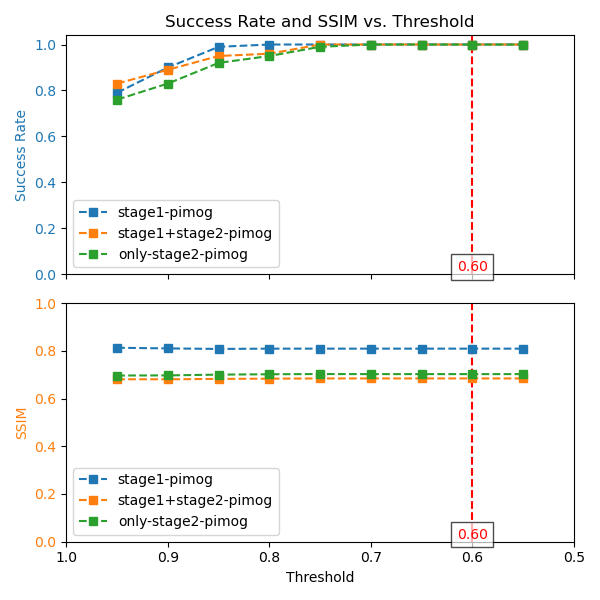} 
    \vspace{-6mm}
    \caption{The detailed success rate of PIMoG forgery attacks and the corresponding SSIM visual metric.}
    \label{fig:spoof-ab-pimog-ssim}
\end{minipage}
\end{figure}

\begin{figure}[!t]
\begin{minipage}{0.48\linewidth}
    \centering
    \includegraphics[width=\linewidth]{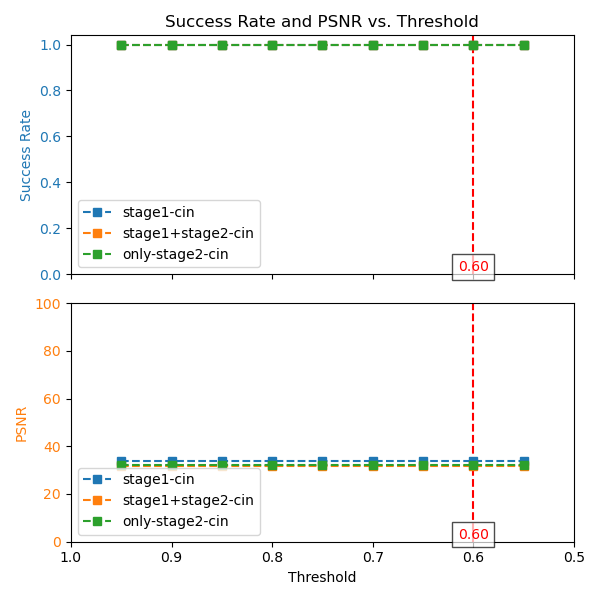} 
    \vspace{-6mm}
    \caption{The detailed success rate of CIN forgery attacks and the corresponding PSNR visual metric.}
    \label{fig:spoof-ab-cin-psnr}
\end{minipage}\hfill
\begin{minipage}{0.48\linewidth}
    \centering
    \includegraphics[width=\linewidth]{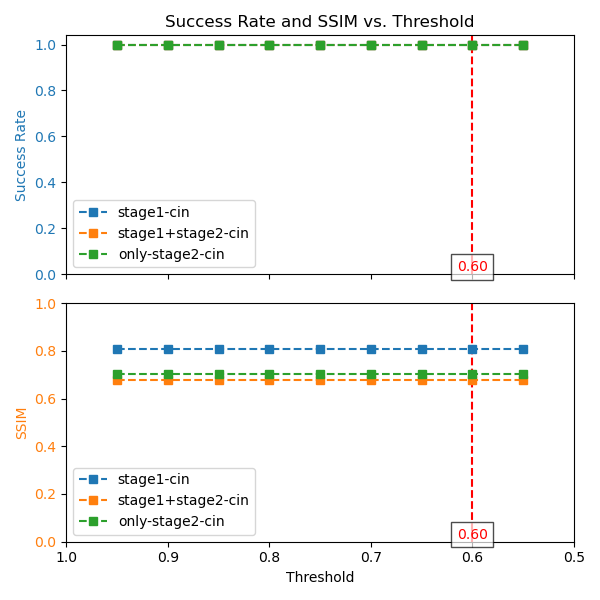} 
    \vspace{-6mm}
    \caption{The detailed success rate of CIN forgery attacks and the corresponding SSIM visual metric.}
    \label{fig:spoof-ab-cin-ssim}
\end{minipage}
\end{figure}


\begin{figure}[!t]
    \centering
    \includegraphics[width=\linewidth]{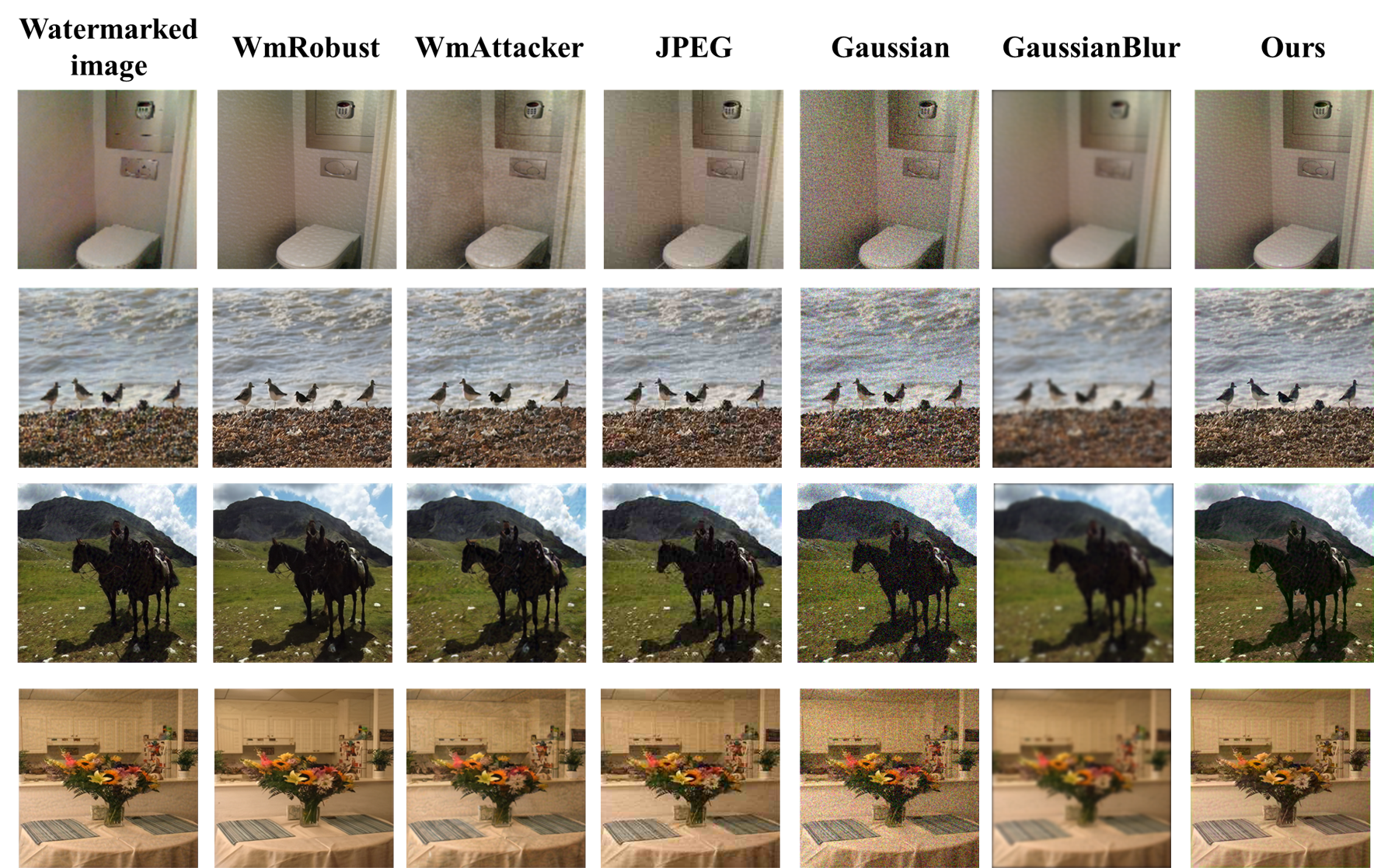} 
   
    \vspace{-3mm}
    \caption{Examples of the evasion attack against HiDDeN.}
    \label{fig:evasion-hidden}
\end{figure}

\begin{figure}[!t]
    \centering
    \includegraphics[width=\linewidth]{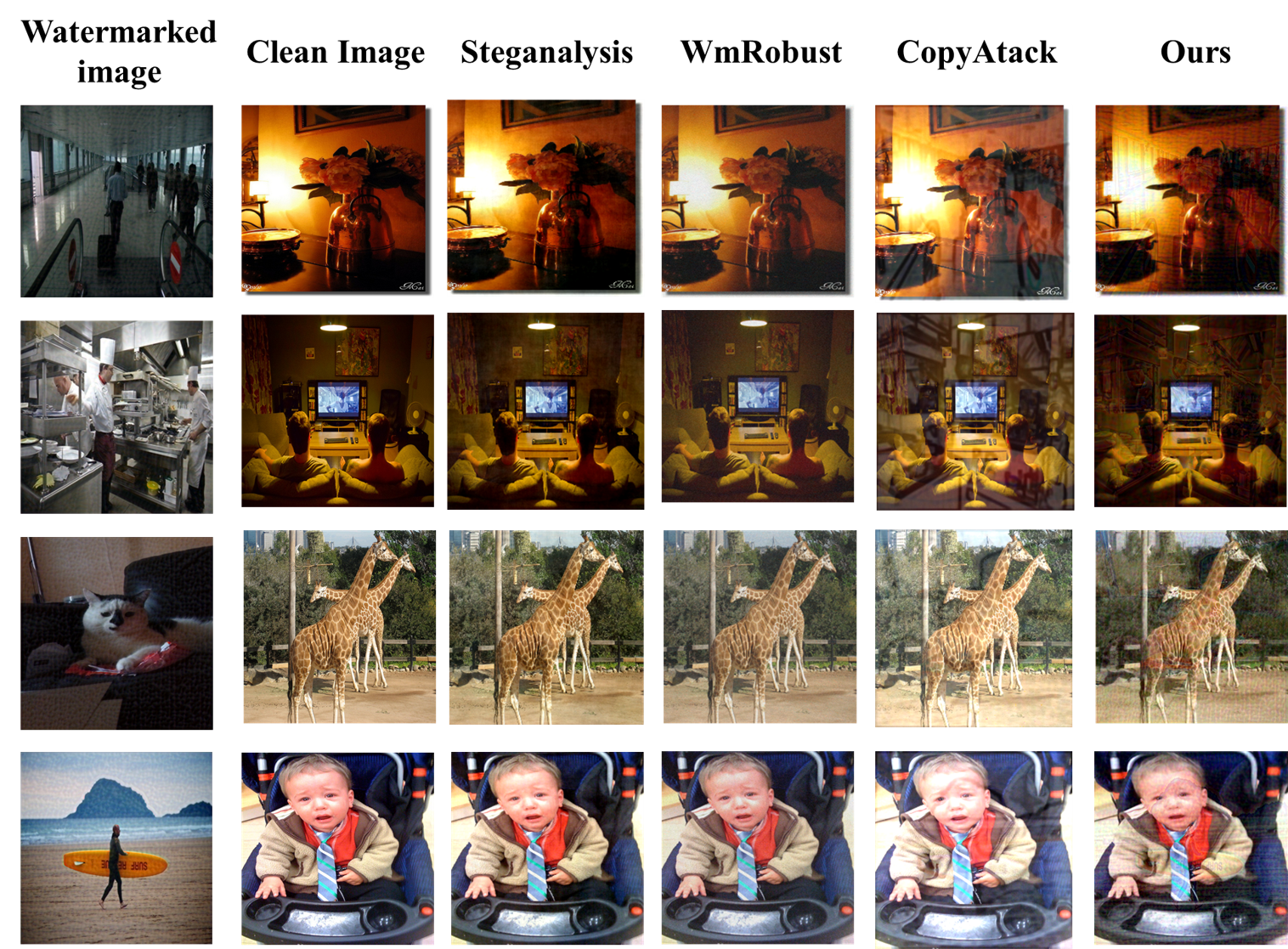} 
    
    \vspace{-3mm}
    \caption{Examples of the spoof attack against HiDDeN.}
    \label{fig:spoof-hidden}
\end{figure}

\begin{figure}[!t]
\begin{minipage}{0.48\linewidth}
    \centering
    \includegraphics[width=\linewidth]{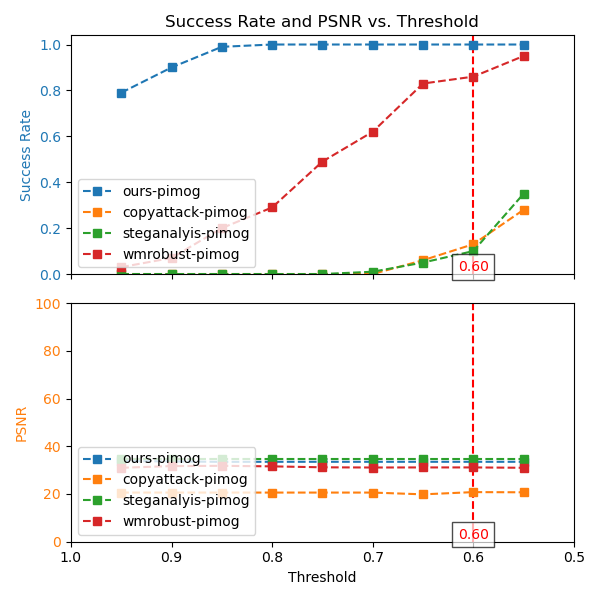} 
   
    \vspace{-3mm}
    \caption{The detailed success rate of PIMoG forgery attacks and the corresponding PSNR visual metric.}
    \label{fig:spoof-related-pimog-psnr}
\end{minipage}\hfill
\begin{minipage}{0.48\linewidth}
    \centering
    \includegraphics[width=\linewidth]{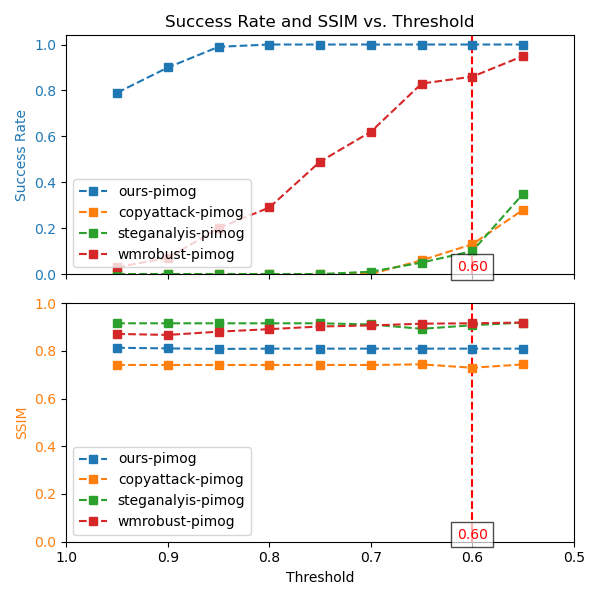} 
    
    \vspace{-3mm}
    \caption{The detailed success rate of PIMoG forgery attacks and the corresponding SSIM visual metric.}
    \label{fig:spoof-related-pimog-ssim}
\end{minipage}
\end{figure}

\begin{figure}[!t]
\begin{minipage}{0.48\linewidth}
    \centering
    \includegraphics[width=\linewidth]{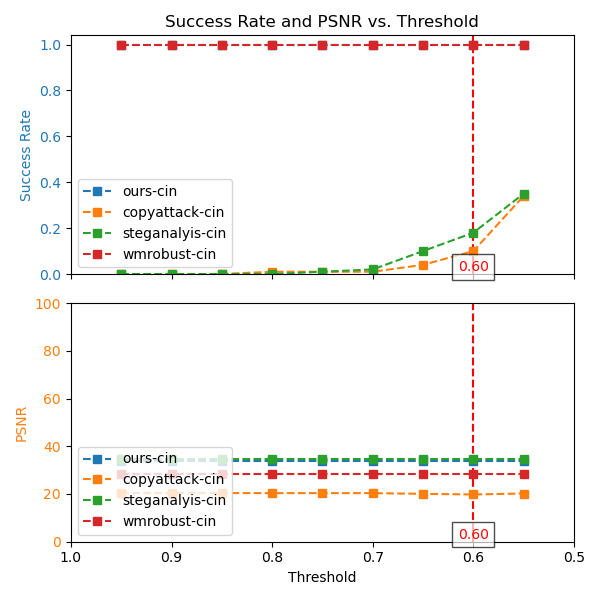} 
   
    \vspace{-6mm}
    \caption{The detailed success rate of CIN forgery attacks and the corresponding PSNR visual metric.}
    \label{fig:spoof-related-cin-psnr}
\end{minipage}\hfill
\begin{minipage}{0.48\linewidth}
    \centering
    \includegraphics[width=\linewidth]{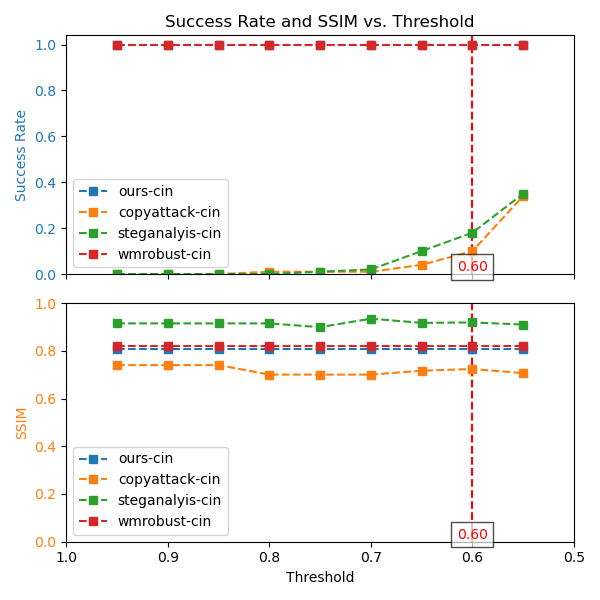} 
    
    \vspace{-6mm}
    \caption{The detailed success rate of CIN forgery attacks and the corresponding SSIM visual metric.}
    \label{fig:spoof-related-cin-ssim}
\end{minipage}
\end{figure}

\begin{figure}[!t]
\begin{minipage}{0.48\linewidth}
    \centering
    \includegraphics[width=\linewidth]{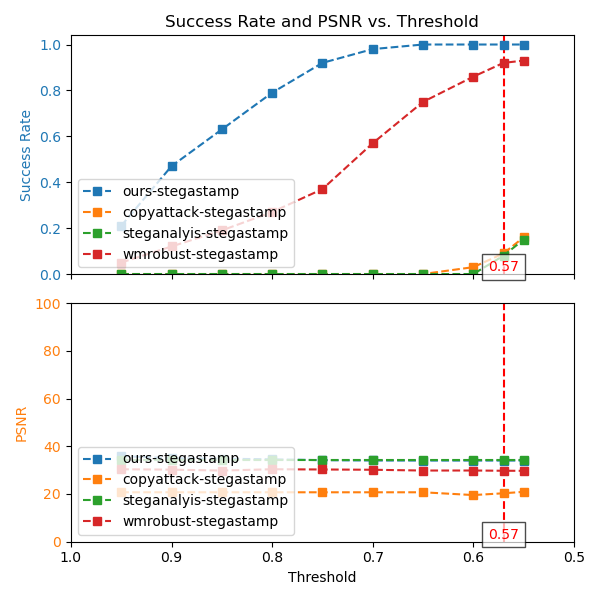} 
    \label{fig:spoof-related-stegastamp-psnr}
    \vspace{-6mm}
    \caption{The detailed success rate of StegaStamp forgery attacks and the corresponding PSNR visual metric.}
    \label{fig:spoof-related-stegastamp-psnr}
\end{minipage}\hfill
\begin{minipage}{0.48\linewidth}
    \centering
    \includegraphics[width=\linewidth]{pics/appendix/spoof-related/pic-stegastamp-evasion-coco-cmp-spoof-PSNR.png} 
    
    \vspace{-6mm}
    \caption{The detailed success rate of StegaStamp forgery attacks and the corresponding SSIM visual metric.}
    \label{fig:spoof-related-stegastamp-ssim}
\end{minipage}
\end{figure}

\begin{figure}[!t]
\begin{minipage}{0.48\linewidth}
    \centering
    \includegraphics[width=\linewidth]{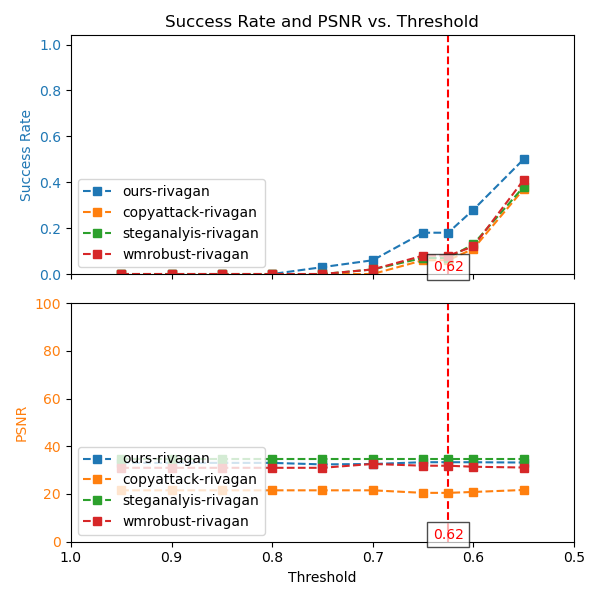} 
    
    \vspace{-6mm}
    \caption{The detailed success rate of RivaGan forgery attacks and the corresponding PSNR visual metric.}
    \label{fig:spoof-related-rivagan-psnr}
\end{minipage}\hfill
\begin{minipage}{0.48\linewidth}
    \centering
    \includegraphics[width=\linewidth]{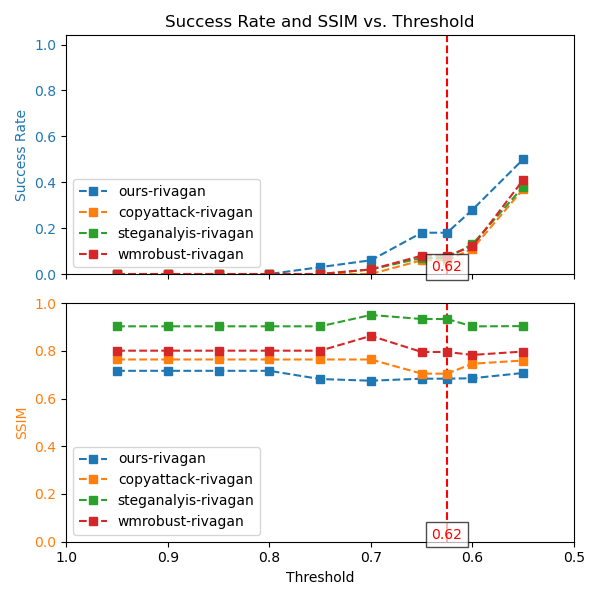} 
    
    \vspace{-6mm}
    \caption{The detailed success rate of RivaGan forgery attacks and the corresponding SSIM visual metric.}
    \label{fig:spoof-related-rivagan-ssim}
\end{minipage}
\end{figure}

\begin{figure}[!t]
\begin{minipage}{0.48\linewidth}
    \centering
    \includegraphics[width=\linewidth]{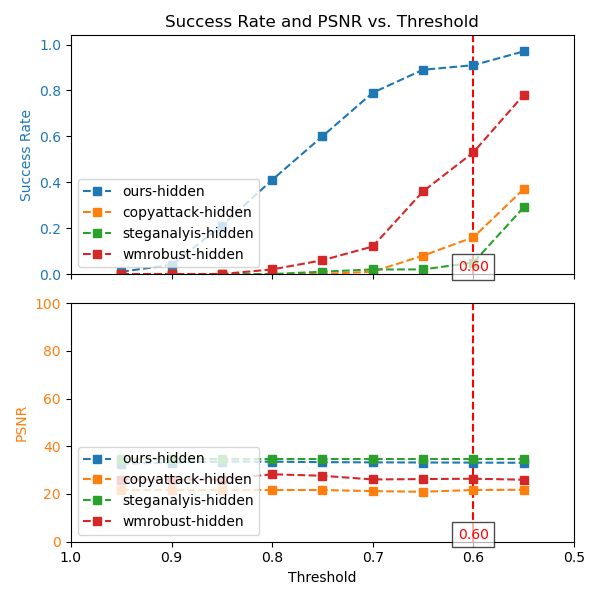} 
    
    \vspace{-6mm}
    \caption{The detailed success rate of HiDDeN forgery attacks and the corresponding PSNR visual metric.}
    \label{fig:spoof-related-hidden-psnr}
\end{minipage}\hfill
\begin{minipage}{0.48\linewidth}
    \centering
    \includegraphics[width=\linewidth]{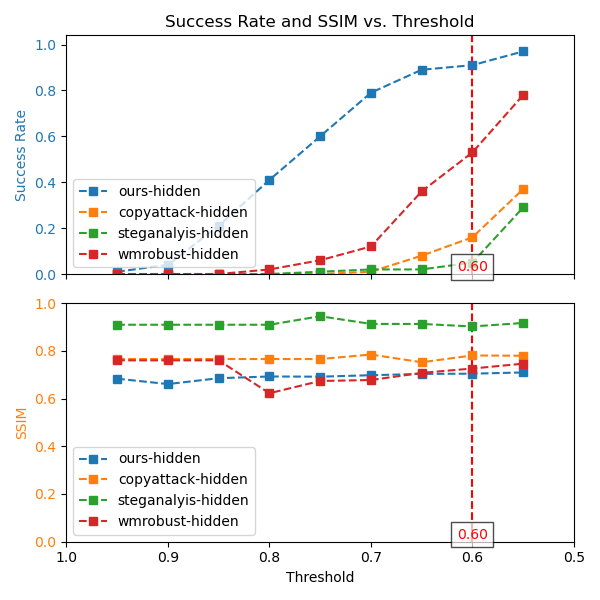} 
    
    \vspace{-6mm}
    \caption{The detailed success rate of HiDDeN forgery attacks and the corresponding SSIM visual metric.}
    \label{fig:spoof-related-hidden-ssim}
\end{minipage}
\end{figure}


\section{Implementation Details}\label{sec:Appendix_Implementation Details}

\begin{algorithm}[tb]
   \caption{Evasion Attack}
   \label{alg:evasion algo}
\begin{algorithmic}
   \STATE {\bfseries Input:} watermarked image $I_{wm}$, step $m$, feature extracter $\mathcal{F}$, clustering algorithm $\mathcal{C}$, optimizer $\mathcal{O}$, objective function $\mathcal{L}$, perturbation budget $\epsilon$
   \STATE {\bfseries Output:} attacked watermarked image $I_a$
   \STATE // find the k clusters with the fewest samples within the cluster, and transform the channel feature indices within the clusters into weights for locating information leakage.
   \STATE $\mathcal{W} \xleftarrow{\mathcal{C}} \mathcal{F}(I_{wm})$ 
   \STATE $\delta \xleftarrow{} noise$
   
   \FOR{$i=1$ {\bfseries to} $m-1$}
   \STATE $\delta \xleftarrow{} \mathcal{O}(\delta,-\nabla_\delta\mathcal{L}(\mathcal{W} \cdot \mathcal{F}(I_{wm}), \mathcal{W}\cdot \mathcal{F}(I_{wm} + \delta))))$
   \IF{$||\delta||_\infty > \epsilon$}
   \STATE  $\delta \xleftarrow{} \delta \cdot \frac{\epsilon}{||\delta||_\infty}$
   \ENDIF
   \ENDFOR
   \STATE {\bfseries return} $I_a = I_{wm} + \delta$
\end{algorithmic}
\end{algorithm}

\begin{algorithm}[tb]
   \caption{Spoof Attack}
   \label{alg:spoof algo}
\begin{algorithmic}
   \STATE {\bfseries Input:} watermarked image $I_{wm}$,  clean image $I'$, step $m$, feature extracter $\mathcal{F}$, clustering algorithm $\mathcal{C}$, optimizer $\mathcal{O}$, objective function $\mathcal{L}$, perturbation budget $\epsilon$
   \STATE {\bfseries Output:} attacked watermarked image $I_a$
   \STATE // find the k clusters with the fewest samples within the cluster, and transform the channel feature indices within the clusters into weights for locating information leakage.
   \STATE $\mathcal{W} \xleftarrow{\mathcal{C}} \mathcal{F}(I_{wm})$

   \STATE $\delta \xleftarrow{} noise$
   \STATE // Stage-\uppercase\expandafter{\romannumeral1}
   \FOR{$i=1$ {\bfseries to} $m-1$}
   \STATE $\delta \xleftarrow{} \mathcal{O}(\delta,\nabla_\delta-\mathcal{L}(\mathcal{W} \cdot \mathcal{F}(I_{wm}), \mathcal{W}\cdot \mathcal{F}(I_{wm} + \delta))))$
   \IF{$||\delta||_\infty > \epsilon$}
   \STATE  $\delta \xleftarrow{} \delta \cdot \frac{\epsilon}{||\delta||_\infty}$
   \ENDIF
   \ENDFOR
    
   \STATE // Stage-\uppercase\expandafter{\romannumeral2}
   \STATE $\delta_s \xleftarrow{} noise$
   \FOR{$i=1$ {\bfseries to} $m-1$}
   \STATE $\delta_s \xleftarrow{} \mathcal{O}(\delta_s,\nabla_{\delta_s}\mathcal{L}((1-\mathcal{W}) \cdot \mathcal{F}(I_{wm}+\delta), (1-\mathcal{W})\cdot \mathcal{F}(I' + \delta_s)))$
   \STATE $\hat{\delta}=-\delta + \delta_s$
   \IF{$||\hat{\delta}||_\infty > \epsilon$}
   \STATE  $\hat{\delta} \xleftarrow{} \hat{\delta} \cdot \frac{\epsilon}{||\hat{\delta}||_\infty}$
   \STATE $\delta_s \xleftarrow{} \hat{\delta} + \delta$
   \ENDIF
   \ENDFOR
   \STATE {\bfseries return} $I_a \in \{I' - \delta, I'-\delta+\delta_s\}$
\end{algorithmic}
\end{algorithm}

\section{Proofs}\label{sec:Appendix_Proofs}

\begin{definition}
A intuitive definition of embeddable threshold:
\begin{gather*}
C(I) = \sup_{W \in \mathcal{P}_1}{\frac{||W||_2}{||I||_2}} \\\\
s.t. PNSR(I, I+W) \ge TV
\end{gather*}
$TV$ represents the lower bound of the visual quality.
\end{definition}

\begin{proposition}
When the robustness requirement exceeds $C(I)$, a decline in visual quality is inevitable.
\end{proposition}

\begin{proof}
Let the distortion layer $\mathcal{T}$ introduce noise $\eta \sim \mathcal{T}$, with the requirement that
$$||wm-\mathcal{D}(I_{wm} + \eta)|| \le \mathcal{B}$$
$\mathcal{B}$ is bit error rate.Then computing channel capacity:
$$R=\frac{1}{2}\log(1+\frac{\epsilon^2||W||^2}{\delta_{\eta}^2})$$

To correctly transmit $K$ bits of information, the following conditions must be met:
\begin{align*}
    R \ge H(wm) = k \Rightarrow \frac{1}{2}\log(1+\frac{\epsilon^2||W||^2}{\delta_{\eta}^2}) \ge H(wm) = k\\
    \Rightarrow \epsilon^2||W||^2_2 \ge (2^{2H(wm)}-1)\delta_{\eta^2}
\end{align*}





according to $C(I) = \sup_{W \in \mathcal{P}_r}{\frac{||W||_2}{||I||_2}} $, 
we get:
$$
\epsilon||W||_2 \le C(I)||I||_2
$$

Let the robustness requirement be $$\epsilon^2||W||^2_2 \ge (2^{2H(wm)}-1)\delta_{\eta^2}$$ 
and the visual quality constraint be $$\epsilon||W||_2 \le C(I)||I||_2$$
When $\sqrt{(2^{2H(wm)}-1)\delta_{\eta^2}} > C(I)||I||_2$, the system cannot simultaneously satisfy both, and it is necessary to increase $C(I)$
\end{proof}



\end{document}